\documentclass{article} %
\usepackage{iclr2024_conference,times}

\usepackage{graphicx, booktabs, multirow}
\usepackage{amsmath, amssymb, mathtools, amsthm}
\usepackage{subcaption}
\usepackage{xcolor}
\usepackage{enumitem}
\usepackage{wrapfig}
\usepackage{tcolorbox}
\usepackage{tabulary}
\usepackage{bm}
\usepackage{pifont}
\usepackage{fontawesome5}
\usepackage[linesnumbered,ruled,vlined]{algorithm2e}
\usepackage{algorithmic}
\usepackage{makecell}

\def\eqref#1{equation~\ref{#1}}

\def\1{\bm{1}}

\def\vx{{\bm{x}}}
\def\vy{{\bm{y}}}

\DeclareMathAlphabet{\mathsfit}{\encodingdefault}{\sfdefault}{m}{sl}
\SetMathAlphabet{\mathsfit}{bold}{\encodingdefault}{\sfdefault}{bx}{n}

\DeclareMathOperator*{\argmax}{arg\,max}
\DeclareMathOperator*{\argmin}{arg\,min}

\definecolor{sailiks-yellow}{RGB}{240,210,90}
\definecolor{sailiks-red}{RGB}{212,88,107}
\definecolor{sailiks-green}{RGB}{170,223,112}
\definecolor{sailiks-blue}{RGB}{106,190,237}
\definecolor{sailiks-orange}{RGB}{249,133,82}
\definecolor{sailiks-gray}{RGB}{106,113,101}

\definecolor{MidnightBlue}{HTML}{006795}
\definecolor{BrickRed}{HTML}{B6321C}
\definecolor{Bittersweet}{HTML}{C04F17}
\definecolor{Green}{HTML}{00A64F}
\definecolor{Emerald}{HTML}{00A99D}
\definecolor{Apricot}{HTML}{FBB982}
\definecolor{Orchid}{HTML}{AF72B0}

\usepackage{hyperref}
\usepackage{url}
\hypersetup{
    colorlinks=true,
    linkcolor=MidnightBlue,
    citecolor=MidnightBlue,
    urlcolor=sailiks-orange
}

\newcommand{\afterfigspace}{\vspace{0pt}}

\theoremstyle{plain}
\newtheorem{theorem}{Theorem}
\newtheorem{definition}{Definition}

\newtheorem{remark}{Remark}

\newtheorem{lemma}{Lemma}[section]

\newcommand{\alg}{\textbf{\texttt{SeRA}}\,}
\newcommand{\sailikstitle}{\textbf{SeRA}: Self-Reviewing and Alignment of LLMs using Implicit Reward Margins}

\title{\sailikstitle}

\definecolor{myorange}{RGB}{255, 165, 0}
\colorlet{myorange}{myorange!90}
\definecolor{myskyblue}{RGB}{135, 206, 235}
\colorlet{myskyblue}{myskyblue!90}

\author{
\qquad Jongwoo Ko${}^{1}$\thanks{
Work done as a research intern at Amazon.
}
\qquad Saket Dingliwal${}^{2}$
\qquad Bhavana Ganesh${}^{2}$
\qquad Sailik Sengupta${}^{3}$
\\[0.2em]
\qquad \qquad \qquad \qquad \qquad \textbf{Sravan Bodapati${}^{2}$} 
\qquad \textbf{Aram Galstyan${}^{2}$}
\\[0.5em]
\qquad \qquad \quad \qquad ${}^{1}$KAIST AI
\qquad \qquad ${}^{2}${\color{sailiks-orange} \faAmazon}\,mazon AGI
\qquad \qquad ${}^{3}${\color{sailiks-orange} \faAmazon}WS AI Labs
}

\iclrfinalcopy %
\begin{document}

\maketitle

\begin{abstract}
Direct alignment algorithms (DAAs), such as direct preference optimization (DPO), have become popular alternatives for Reinforcement Learning from Human Feedback (RLHF) due to their simplicity, efficiency, and stability. However, the preferences used in DAAs are usually collected before the alignment training begins and remain unchanged (off-policy). This design leads to two problems where the policy model (1) picks up on spurious correlations in the dataset (as opposed to learning the intended alignment expressed in the human preference labels), and (2) overfits to feedback on off-policy trajectories that have less likelihood of being generated by the updated policy model.
To address these issues, we introduce Self-Reviewing and Alignment (\textbf{\texttt{SeRA}}), a cost-efficient and effective method that can be readily combined with existing DAAs. \alg~comprises of two components: (1) {\em sample selection} using implicit reward margins, which helps alleviate over-fitting to some undesired features, and (2) {\em preference bootstrapping} using implicit rewards to augment preference data with updated policy models in a cost-efficient manner. Extensive experimentation, including some on instruction-following tasks, demonstrate the effectiveness and generality of \alg~in training LLMs on offline preference datasets with DAAs.
\end{abstract}

\section{Introduction}

Large Language models (LLMs;~\citealt{achiam2023gpt, team2023gemini, jiang2024mixtral}) have shown mastery on a multitude of tasks in artificial intelligence (AI), ranging from creative writing~\citep{wang2024weaver} to code generation~\citep{li2023starcoder}, and mathematical reasoning~\citep{ahn2024large}. With success, comes concerns related to their safety, reliability, and potential for misuse in sensitive domains like social manipulation, cyber-attacks, etc. To address some of these challenges, works have considered aligning LLMs to human values/preferences using approaches like Reinforcement Learning from Human Feedback (RLHF;~\citealt{ouyang2022training, bai2022training}). 

Initial RLHF approaches, such as Proximal Policy Optimization (PPO), necessitated the need for two separate models— a policy model (the LLM model to be aligned) and a reward model \citep{schulman2017proximal,ouyang2022training}.  Such approaches often result in challenges related to stability and scalability~\citep{rafailov2024direct,zhao2023slic,meta2024llama3}. To mitigate these shortcomings, Direct Alignment Algorithms\,(DAAs), such as direct preference optimization~(DPO; \citealt{rafailov2024direct}), sequence likelihood calibration with human feedback~(SLiC-HF; \citealt{zhao2023slic}), and identity policy optimization~(IPO; \citealt{azar2024general}), have emerged as popular alternatives. In DAA, we directly update the policy model/LLM using (some closed-from solution of) the pairwise preference data without the need for an explicit reward model, making the alignment process simpler, more efficient and stable compared to earlier methods~\citep{rafailov2024direct}.

However, DAAs leverage preference rating on off-policy trajectories collected prior to alignment tuning (at times, generated by a different LLM/policy-model \citep{zhao2023slic, tunstall2023zephyr}), resulting in two challenges.
\textit{First}, the off-policy preference data, which still require labor-intensive annotations, often contain noisy features orthogonal to the true preference objective (e.g. longer responses are more preferred; \citealt{park2024disentangling}). Thus, training on this data can teach the policy model to over-fit to such noise~\citep{mitchell2023note, chowdhury2024provably}, thereby learn spurious correlations~\citep{park2024disentangling, rafailov2024direct}.
While some works try to consider explicit regularization to alleviate over-fitting to spurious features~\citep{park2024disentangling}, identifying all possible spurious features is often challenging and incomplete~\citep{rafailov2024scaling}. 
\textit{Second}, the preference feedback, being off-policy, cannot aid the policy model to obtain feedback on its own generations during alignment training. As the off-policy and the on-policy may not belong to the same distribution, this indirect feedback can inhibit improvement of the policy model~\citep{tajwar2024preference}.
The latter shortcoming is partially addressed by PPO-like methods, where the stand-in reward model can rate on-policy generations~\citep{guo2024direct}, although its robustness on on-policy trajectories remain questionable.

To inherit the best of both worlds, several works have explored using RL from AI-Feedback (RLAIF; \citealt{guo2024direct,rosset2024direct,lee2023rlaif,bai2022constitutional}). These approaches aim to mimic the human's preference rating behavior using high-quality LLMs (via API access). Unfortunately, the approximate preference-distillation-approaches are both inefficient and expensive. For example, Direct Nash optimization~(DNO;\,\citealt{rosset2024direct}) incurs a cost of \$34,000 to obtain preference annotations using GPT-4~\citep{zheng2024judging}. Thus, the motivations of bypassing training of a reward model lands up eventually incurring exorbitant costs.

\paragraph{Contributions.} In this paper, we propose to augment DAAs with a novel strategy, \textbf{Se}lf-\textbf{R}eviewing and \textbf{A}lignment (\alg) that uses Implicit Reward Margin (IRM; defined as \autoref{eqn:rm}) for off-policy sample selection and iteratively bootstraps preference data for alignment (see \autoref{fig:sera}). Specifically, we show that:
\begin{itemize}[leftmargin=*, itemsep=0pt]
    \item \textbf{IRM-based Off-policy Sample Selection} helps mitigate over-optimization of policy models to spurious correlations (such as considering response length to gauge preference). Its efficacy and cost-effectiveness is consistent across various DAAs, datasets, problems, and model variants.
    \item \textbf{IRM-based Preference Data Bootstrapping} mitigates DAAs from continuously updating policy models with off-policy data and proposes a decoding and rejection sampling approach to extract informative policy pairs (based on IRM) for continual alignment training. This improves the efficacy and efficiency of DAAs without the need for expensive (\& external) reward models.
    \item \textbf{Better Performance and Versatility: } We empirically showed that \alg can be widely used across various DAAs (\textit{e.g.,} DPO, IPO, SLiC-HF, SimPO) and on various LLMs~(\textit{e.g.}, TinyLlama-1.1B, Pythia-2.8B, Mistral-7B) consistently outperforming SoTA baselines~\citep{kim2024sdpo, pattnaik2024curry}. Finally, we conduct an exploratory analysis and ablations to better understand \alg.
\end{itemize}

\begin{figure}
    \centering
    \includegraphics[width=\linewidth]{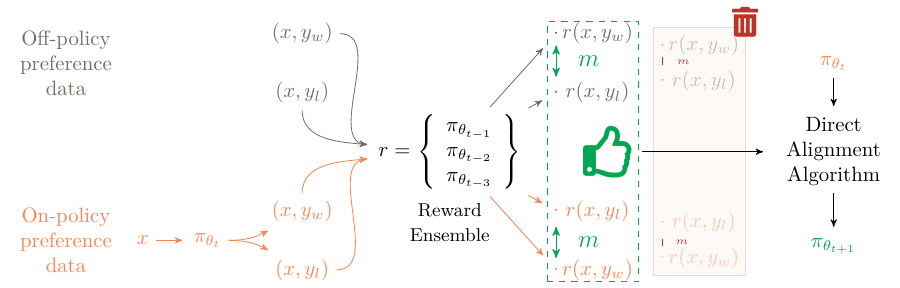}
    \caption{\alg\ uses an iterative policy-model ensemble to obtain reward margins between preferred and dis-preferred trajectories in {\color{sailiks-gray} off-policy} and {\color{sailiks-orange} on-policy} preference data. It then selects preference data with {\color{Green} large} reward margins and rejects the data with {\color{BrickRed!90} smaller} rewards margins to update the policy models via Direct Alignment Algorithms (DAAs).}
    \label{fig:sera}
\end{figure}

\section{Background}
In this section, we provide a brief overview of related works~(\S\ref{sec:related}) and discuss some preliminaries~(\S\ref{sec:preliminary}) necessary to understand our contribution.

\subsection{Related Work}\label{sec:related}

\paragraph{Alignment with preference data} Actor-critic RLHF frameworks~\citep{christiano2017deep, stiennon2020learning, bai2022training, ouyang2022training} seeks to align language models to human preferences, but is both memory-intensive (requires a policy model and a reward model to be on device simultaneously) and unstable during training. 
To mitigate this, several algorithms, such as direct preference optimization (DPO; \citealt{rafailov2024direct}) and sequence likelihood calibration (SLiC-HF; \citealt{zhao2023slic}), learn the contrastive preference in the offline setting using a closed-form loss function without the need for an explicit critic/reward model.
\cite{azar2024general} argued that without regularization, a policy can easily overfit to deterministic preferences and introduced identity preference optimization (IPO) to directly optimize offline preference probabilities with regularization.
In parallel, the availability of offline preference datasets like UltraFeedback~\citep{cui2023ultrafeedback}, OpenOrca~\citep{OpenOrca}, Helpfulness and Harmless~\citep{ganguli2022red}, and TL;DR~\cite{stiennon2020learning} has made aligning LLMs more accessible.

\paragraph{Alignment with on-policy preference annotations.} Despite its effectiveness, aforementioned approaches use reward signals over off-policy trajectories (or critic models trained to reward off-policy trajectories) that may have a distribution mismatch from on-policy trajectories generated by the policy model, which is improved in an iterative manner \citep{guo2024direct, ko2024distillm}. A major impediment is the lack of human annotators (gold critic) to provide preference data for trajectories generated by iteratively-improving policy models. To address this, \citet{guo2024direct} proposed online AI-feedback~(OAIF) by using another LLM to annotate which of two online-sampled outputs from the current policy is preferred. This idea of distilling preference labels comes in various flavors. For example, Direct Nash optimization~(DNO) distills preference signals for on-policy trajectories generated by the updated policy model using a strong teacher model \citep{rosset2024direct}. While effective at distilling preference, this methods can incur high expenses for calling proprietary LLM APIs or violate (legal) terms-of-use, making them difficult to use practically.

Along these lines, cost-efficient alternatives encourage using the policy models itself as preference labelers or leveraging implicit preference superiority with high-quality human responses. Self-rewarding LMs~\citep{yuan2024self} studied the benefits of iteratively training on preferences derived from recent policy's sampled outputs. However, in their work, they used the LLM itself as the annotator based on prompting~\citep{zheng2024judging}, which is only valid for LLMs capable of following the prompt properly.
Self-play fine-tuning~(SPIN; \citealt{chen2024self}) and adversarial preference optimization~(APO; \citealt{cheng2023adversarial}) are both iterative LLM training techniques that are compatible with contrastive losses. However, as highlighted in~\citep{rosset2024direct}, expecting high-quality predefined responses being better than the policy generation without considering of annotator feedback is a severe limitation. In general, methods that use a proxy reward model (whether a more capable LLM teach or the policy model itself) can suffer from reward-hacking that can significantly degrade the alignment quality in iterative improvement of policy models \citep{pan2024spontaneous}.

\paragraph{Sample Selection Approaches.}  In this work, \alg~partially relies on selecting off-policy samples based on an Internal Reward Margin (IRM) of policy models, as opposed to using the absolute reward. We highlight that this minimize distribution shift and, in turn, reduces performance deterioration when using offline preference datasets. While previous works have considered different mechanisms to sub-sample existing data, at times reducing the distributional shift across iterations, their selection approaches and goals differ from ours. For example, \citet{cazenavette2022dataset} and \citet{chai2023goodcore} focused on improving training efficiency by using a sub-set of the data while guaranteeing performance similar to a model trained on the full data. Along similar lines, \citet{kim2021fine, xia2022sample} and \citet{mirzasoleiman2020coresets} formulate the problem of finding a a core subset of the entire training data that can capture the statistical properties of the original dataset, also aiming to enhance training efficiency. Recent work has also investigated using a weighted loss function over off-policy samples based on the recent policy model's judgements of how close they are to on-policy trajectories \cite{zhou2024wpo}. Instead of focusing on improving training efficiency or finding ways to adapt off-policy data to an (pseudo) on-policy settings, our motivation is to eliminate preference pairs that can cause performance degradation due to distribution shift.  Inspired by work on sample selection for robust training in the presence of noisy labels~\citep{kim2021fine, ko2023gift, ahn2023fine}, our approach aims to select informative samples that clearly align with human preferences based on IRM.

\paragraph{Training with Self-Generated Response.} Another component of \alg~ is exploiting self-generated responses during the training procedure. In this regard, \citet{agarwal2024policy} proposed an on-policy distillation approach that uses student-generated outputs as input for both student and teacher models to reduce distribution shift between training and inference. In contrast, our method relates more to bootstrap techniques~\citep{tibshirani1993introduction} that have policy models produce preference pairs from given queries. \citet{grill2020bootstrap} introduced a self-supervised learning framework with two types of networks—online and target—where the online network predicts the same as the target network on different augmented views. \citet{li2022blip} introduced BLIP to pre-train vision-language models by effectively utilizing noisy web data and bootstrapped captions.

Concurrent to our work, \citet{chen2024bootstrapping, kim2024aligning} suggested constructing preference samples by training policy LLMs without using external reward models. However, while these works did not explore other DPO variants, such as SLiC-HF and IPO, our work demonstrated their effectiveness on various tasks. Additionally, we provide the statistical grounding of our choice of IRM and empirically enhance the method by introducing an ensemble of different iterations to improve the effectiveness of our IRM-based technique.

\subsection{Preliminaries}\label{sec:preliminary}
\paragraph{Reinforcement Learning From Human Feedback.}
The goal of RLHF is to optimize the policy LLM $\pi_{\theta}$ such that it maximizes the expected value of the reward function. A common approach to modeling the reward function is using Bradley-Terry~(BT; \citealt{bradley1952rank}) model:
\begin{equation}
    p(\vy_{w} \succ \vy_{l}|\vx) = \frac{\exp \left( r(\vx, \vy_w) \right)}{\exp \left( r(\vx, \vy_w) \right) + \exp \left( r(\vx, \vy_l) \right)} = \sigma( r(\vx, \vy_w) - r(\vx, \vy_l)),
\end{equation}
where $\vy_w$ (and $\vy_l$) denotes the preferred/winning (and losing) policy, $p$ denotes the preference distribution that approximates an unobserved latent reward $r(\vx, \vy)$, and $\sigma$ is the logistic function. To achieve this, RLHF first trains a reward model $r_\phi(\vx, \vy)$. Then, RLHF updates $\pi_\theta$ with an on-policy RL algorithm like PPO~\citep{schulman2017proximal}, iteratively optimizing the model to provide responses more preferred by human. The most common objective is
\begin{equation}\label{eq:rlhf}
    \mathcal{L}_{\text{RLHF}} = \mathbb{E}_{\vx \sim \mathcal{D}, \vy \sim \pi_{\theta}(\cdot|\vx)} \left[ r_\phi (\vx, \vy) \right] - \beta \mathbb{D}_{\text{KL}} \left[ \pi_\theta (\vy|\vx) \| \pi_{\text{ref}} (\vy|\vx) \right],
\end{equation}
which enforces a KL divergence~\citep{kullback1951information} penalty with a reference distribution $\pi_{\text{ref}}(\vy|\vx)$, to prevent the LLM $\pi_\theta$ from straying too far from its initialization. The hyper-parameter $\beta$ balances between exploiting the reward function and deviating from $\pi_{\text{ref}}(\vy|\vx)$.

\paragraph{Direct Alignment Algorithms.} RLHF is computationally expensive and unstable~\citep{rafailov2024direct,meta2024llama3}. Thus, many algorithms~\citep{rafailov2024direct, zhao2023slic, azar2024general} have been proposed to overcome these challenges.
A common idea is to analytically derive the optimal policy and parameterize it using the reward function from \autoref{eq:rlhf}.
In DPO~\citep{rafailov2024direct}, the optimal policy $\pi^{*}$ under the BT model satisfies:
\begin{equation}\label{eqn:prob}
    p^{*}(\vy_{w} \succ \vy_{l}|\vx) = \left( 1 + \exp \left( \beta \log \frac{\pi^{*}(\vy_{l}|\vx)}{\pi_{\text{ref}}(\vy_{l}|\vx)} - \beta \log \frac{\pi^{*}(\vy_{w}|\vx)}{\pi_{\text{ref}}(\vy_{w}|\vx)} \right) \right)^{-1},
\end{equation}
where $p^{*}(\vy_{w} \succ \vy_{l}|\vx)$ is underlying true preference, the probability that $\vy_w$ is more preferred than $\vy_l$. With human preference data expressed in terms of the optimal policy rather than the reward model, we can formulate a maximum likelihood objective for a parameterized policy $\pi_{\theta}$. 
\begin{equation}\label{eqn:dpo}
    \mathcal{L}_{\text{DPO}}(\vx, \vy_{w}, \vy_{l}) = - \log \sigma \left( \beta \log \frac{\pi_{\theta}(\vy_{w}|\vx)}{\pi_{\text{ref}}(\vy_{w}|\vx)} - \beta \log \frac{\pi_{\theta}(\vy_{l}|\vx)}{\pi_{\text{ref}}(\vy_{l}|\vx)} \right)
\end{equation}
Current methods for LLM alignment first collect a dataset of pairwise preferences by obtaining two responses to an input prompt $\vx$ generated using an LLM. Then, human or AI annotators rank these responses, yielding a preferred response $\vy_{w}$ and a less preferred one $\vy_{l}$.

\section{Our Method: Self-Reviewing and Alignment}

In contrast to recent works on DAAs~\citep{rosset2024direct, guo2024direct}, we propose \alg~to improve alignment with DAA without the need for any external supervision. While self-verification highlights that a policy model can dual up as a reward model \citep{weng2022large}, we show that this self-supervision can be done implicitly for alignment using the notion of Implicit Reward Margin (IRM). After defining IRM in this section, we describe (1) an IRM-based off-policy sample selection, and (2) an IRM-based iterative preference data bootstrapping. Finally, we describe our approach \alg~and how it can be seamlessly incorporated into DAAs.

\subsection{Implicit Reward Margin (IRM)}
While works have shown that alignment datasets may have ambiguous preferences~\citep{yang2023rlcd, chowdhury2024provably}, even the use of unambiguous preference data can result in policy models unintentionally learning spurious correlations (e.g. on response length)~\citep{park2024disentangling, rafailov2024scaling}. Recent works consider selecting preference samples $(\vx, \vy_{w}, \vy_{l})$ where there exists a large difference in preferences (i.e. $\vy_w \succ\succ \vy_l$) to help mitigate the noise in the learning~\citep{yang2023rlcd, rosset2024direct}.
We leverage this observation and introduce a Implicit Reward Margin (IRM) to quantify the difference between the preference over two responses using the policy model $\pi_\theta$ itself. Formally, we define IRM as follows where $\pi_{\text{ref}}$ is the reference model: 
\begin{equation}\label{eqn:rm}
    m(\vx, \vy_{w}, \vy_{l}) \coloneqq \frac{1}{\beta}\left(r(\vx, \vy_{w}) - r(\vx, \vy_{l})\right) = \log \frac{\pi_{\theta}(\vy_{w}|\vx)}{\pi_{\text{ref}}(\vy_{w}|\vx)} - \log \frac{\pi_{\theta}(\vy_{l}|\vx)}{\pi_{\text{ref}}(\vy_{l}|\vx)},
\end{equation}
which follows from \autoref{eqn:prob} and \autoref{eqn:dpo}. Unlike \citet{rafailov2024direct}, we omit the normalization term for the reward $r(\vx, \vy) \coloneqq \log \left(\frac{\pi_\theta(\vy|\vx)}{\pi_{\text{ref}}(\vy|\vx)}\right)$ for tractability.

\subsection{IRM-based Off-policy Sample Selection}
\label{sec:reward_selection}
\begin{figure}
    \centering
    \includegraphics[width=\textwidth]{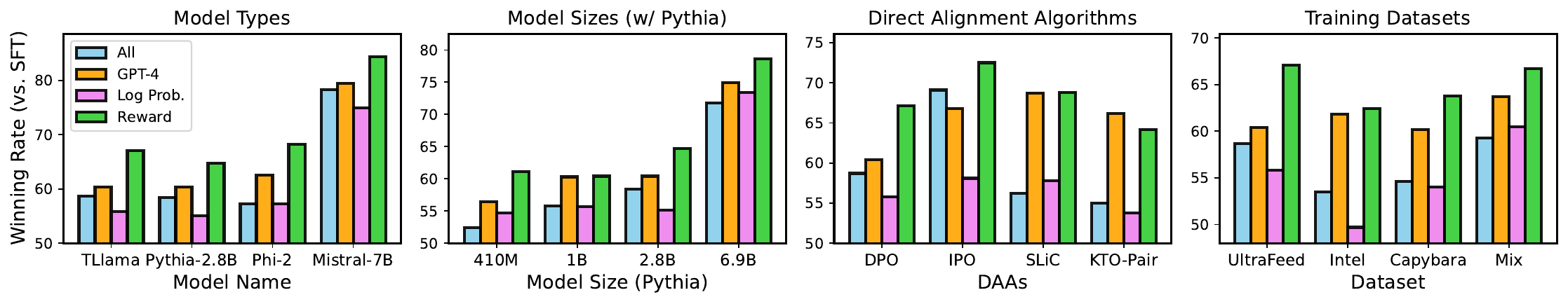}
    \caption{Motivation of \alg: The reward-margin based selection showed higher win-rates (measured by LLM-as-a-Judge~\citep{zheng2024judging} with Claude 3~\citep{claude3}) against the base SFT model across a variety of settings-- (a) Model types: TinyLlama-1.1B~\citep{zhang2024tinyllama}, Pythia-2.8B~\citep{biderman2023pythia}, Phi-2~\citep{textbooks2}, and Mistral-7B~\citep{jiang2023mistral}. (b) Model size: various model sizes of Pythia~(410M, 1.0B, 2.8B, 6.9B). (c) DAAs: DPO~\citep{rafailov2024direct}, IPO~\citep{azar2024general}, SLiC~\citep{zhao2023slic}, and KTO-Pair~\citep{ethayarajh2024kto} and (d) training datasets. 
    }
    \vspace{-10pt}
\label{fig:reward_motif}
\afterfigspace
\end{figure}

We note that a policy model  $\pi_\theta$ can generate mis-calibrated rewards for individual trajectories \citep{panickssery2024llm}, but we empirically study if the reward margin between preferred ($\vy_w$) \textit{vs.} less-preferred trajectories ($\vy_l$) can provide relevant signals to choose samples for improving alignment. Finally, we propose to filter out training samples with {\em lower} IRMs. To motivate our choice, we compare against other sample selection methods in the \autoref{fig:reward_motif}. 
 IRM-based sample selection (in {\color{green!70!black} green}) consistently leads to higher win rates against the SFT model, when compared to alignment on the entire preference samples (in {\color{myskyblue!90!black} sky-blue}), or strong  baseliense that using RLAIF with GPT-4 score~(\cite{zheng2024judging}; in {\color{orange} orange}) or the log probability of $\pi_{\text{ref}}$~(\cite{pattnaik2024curry}; in {\color{pink!90!black} pink}) for sample selection. These gains hold true across model types, model size, DAAs, and datasets. Moreover, methods that use GPT-4 to obtain reward signals for sampling \citep{rosset2024direct,zheng2024judging} incur high-cost; in comparison, our sample selection strategy that leverages $\pi_\theta$ to compute IRM is inexpensive and more effective.

\paragraph{Mitigating over-optimization to spurious correlation.}
To understand why IRM selection benefits alignment, we first highlight that recent works have shown that DAAs prioritize features of the data based on their complexity and prevalence, often resulting in learning preference signals based on dimensions orthogonal to the alignment objective, such as the length of responses~\citep{park2024disentangling, rafailov2024scaling}. Thus, we empirically analyze if IRM's sample selection strategy can weed out such noisy examples, thereby preventing over-optimization. 
To this extent, we study various correlations on TinyLlama on the UltraFeedback dataset for the four aforementioned sample selection strategies.

In ~\autoref{fig:over_optimize}, we can observe that the rewards  of the model trained on IRM-based sample selection (fourth column) has the highest $R^2$ scores with GPT-4 score, while also having the lowest $R^2$ score with the response lengths. 
This indicates that the policy model is highly aligned with the preferences of GPT-4 but does not learn the spurious feature of response length. Even though the rewards of the model without any sample selection (first column) also have the highest correlation with GPT-4 score, their correlation with response length is also high, a result consistent with previous works~\citep{park2024disentangling, rafailov2024scaling} highlighting that DAAs indeed risks over-optimization where the model can learn spurious correlations.
These results showcase that the proposed IRM selection strategy can mitigate over-optimization problems to response length, while other selection methods such as GPT-4 (second column; \citealt{rosset2024direct})\footnote{This result is consistent with recent works indicating that even GPT-4 evaluation has a length bias\,\citep{wu2023style, koo2023benchmarking, dubois2024length}. This problem might also affect off-the-shelf reward models based on existing work on reward hacking of reward models~\citep{gao2023scaling}, as shown in \autoref{tab:tldr}.} and log probability of the reference model (third column; \citealt{pattnaik2024curry}) don't mitigate or, at times, exacerbate the undesired bias.

\begin{figure}
    \centering
    \includegraphics[width=\textwidth]{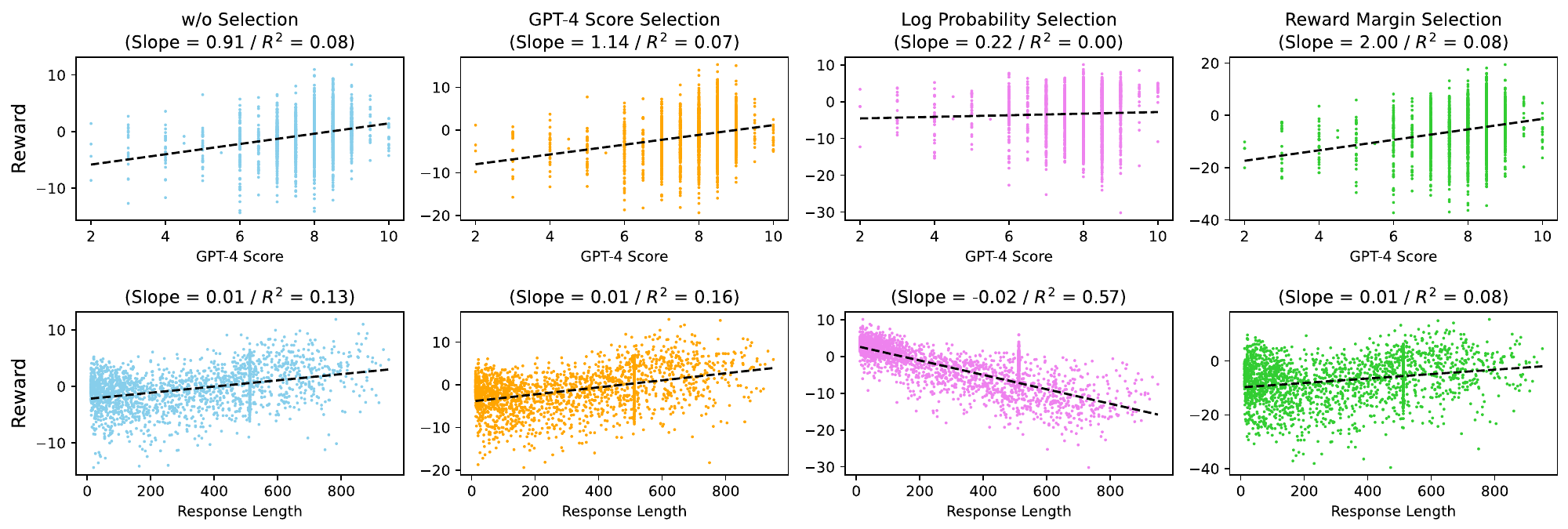}
    \caption{
        Comparison of correlations between the rewards of the trained policy model and features (\textit{e.g.,} GPT-4 score and response length) across different selection methods (\textit{i.e.,} no selection, GPT-4 score-based, log probability of the reference model, and IRM-based selection). 
        \textbf{[Row 1]} Correlation between GPT-4 Score \& implicit reward~(\textit{i.e,} $r(\mathbf{x}, \mathbf{y}_{w})$) for $\mathbf{y}_w$.
        \textbf{[Row 2]} Correlation between response length (\textit{i.e.} $|\mathbf{y}_w|$) and the implicit reward for chosen responses.
        The model with IRM selection~(\textit{i.e.} \textbf{[Column 4]}) shows the highest $R^2$ score for the first row, but the lowest $R^2$ score for the second row. This indicates that IRM-based selection strategy can effectively mitigate the over-optimization on response length~\citep{park2024disentangling}. 
    }
\label{fig:over_optimize}
\afterfigspace \afterfigspace \afterfigspace \afterfigspace
\end{figure}

\paragraph{Theoretical Intuition of Sample Selection.} Research has shown that DNNs first learn explicit features and then memorize minor ones~(\textit{e.g.,} spurious features, noisy labels) during the later training stage~\citep{zhang2017understanding, liu2020early, nam2020learning, liu2021just}. 
Inspired by this, we intend the policy LLM 
to first learn the human preference using confident preference pairs where true preferences $p^*(\vy_{w} \succ \vy_{l}|\vx) \simeq 1$. 
As the estimated probability $\hat{p}_{\theta}(\vy_{w} \succ \vy_{l}|\vx) = \sigma \left( \beta \log \frac{\pi_{\theta}(\vy_{w}|\vx)}{\pi_{\text{ref}}(\vy_{w}|\vx)} - \beta \log \frac{\pi_{\theta}(\vy_{l}|\vx)}{\pi_{\text{ref}}(\vy_{l}|\vx)} \right)$ is modeled by IRM in Eqn.\,(\ref{eqn:rm}) using the BT model, IRM selection removes samples with small margins that might lead to over-optimization on spurious features or noisy preferences.

The inherent problem of DPO is that is treats all preference samples equally (\textit{i.e.,} $p(\vy_w \succ \vy_l|\vx)=1, \forall (\vx, \vy_w, \vy_l) \in S$) even when the true preferences are not, i.e. $p^{*}(\vy_w \succ \vy_l|\vx) \neq p^{*}(\vy'_w \succ \vy'_l|\vx')$. By minimizing the empirical risk on such datasets, the difference between $\hat{p}$ and $p^{*}$ increases, which can results in over-fitting to spurious features.
However, by training LLMs on confident samples, we can achieve lower empirical risk and less overfitting~(\textit{i.e.}, small differences between $\hat{p}_{\theta}$ and $p^{*}$). Thus, we can successfully learn human preferences without learning spurious features, as shown in Fig.~\ref{fig:over_optimize}. While this analysis is consistent with the argument in \citet{azar2024general}, we formally calibrate the upper-bound for this risk as follows.
\begin{remark}[\textbf{Informal Statement for \autoref{thm:main}}]
    Under DPO training, with probability at least $1-\delta$, the upper bound of the risk of $f(\vx, \vy_w, \vy_l) \coloneqq r(\vx, \vy_w) - r(\vx, \vy_l)$, denoted as $R(f)$, over the underlying distribution $\mathbb{P}$, where the sample set $S = { (\vx, \vy_w, \vy_l) } \sim \mathbb{P}$, can be defined as:
    \begin{equation*}
        R(f) \leq {\color{Emerald!70!black} \hat{R}(f; S)} + \mathcal{O} \left( {\color{Apricot!70!black} \mathbb{E}_{\mathbb{P}} \| \hat{p}_{\theta}(\cdot) - p^{*}(\cdot) \|_{2}} \right) + \mathcal{O} \left( \sqrt{\tilde{\mathbb{V}}_{|S|}(f) \cdot \frac{\log (\mathcal{M}_{|S|}/\delta)}{|S|}} + \frac{\log ({\mathcal{M}_{|S|}}/{\delta})}{|S|} \right),
    \end{equation*}%
where $\mathcal{M}_{|S|}$ is uniform covering number. This reflects the risk of classifying preferable responses is upper bounded by {\color{Emerald!70!black} the empirical risk over observed samples}, the {\color{Apricot!70!black} difference between estimated and true preferences} and an additional term that is zero when $|S| \rightarrow \infty$. (See \autoref{app:proof} for the proof.)
\end{remark}

\subsection{IRM-based Preference Data Bootstrapping}
\label{sec:preference_bootstrap}
As mentioned earlier, DAA's use of off-policy preference annotation suffers from distributional mismatch between the sequences observed during training (which are from different LLMs ahead of training) and those generated by the iteratively updated policy LLM~\citep{arora2022exposure, ko2024distillm}, leading to reduced efficacy of DAAs~\citep{guo2024direct}. To address this, we present preference data bootstrapping that considers a decoding-time strategy to sample candidate pairs from the (updated) policy LLM, followed by a rejection sampling of pairs with low IRM.

To build a preference pair, we sample the $R$ ($\geq 2$) distinct candidate responses $\vy^{(i)}_{1}, \ldots, \vy^{(i)}_{R} \sim \pi_{\theta_{t-1}}(\cdot|\vx^{(i)})$ from a query $\vx^{(i)}$ by using decoding-time sampling. We then compute the implicit reward for each response $j (\in \{1,\dots,R\})$ using the term $r(\vx^{(i)}, \vy^{(i)}_{j})$ in \autoref{eqn:rm} and select the pair that maximizes the IRM.

Our method bears similarity to self-rewarding LMs (SRLM; \citealt{yuan2024self}) that reward responses by using  LLM-as-a-Judge\,\citep{zheng2024judging}. Unfortunately, this approach works reliably only when one uses LLMs with sufficient instruction-following capability (see \autoref{qualitative:srlm}), which leads to significant costs and limited applicability for smaller LLMs. In contrast, our approach is versatile across LLMs with a wide range of capacities and boasts strong empirical efficacy (see \autoref{tab:srlm}).

\subsection{\textbf{\texttt{SeRA}}: Self-Reviewing and Alignment}\label{sec:sera}

\begin{algorithm}[t]
    \caption{Self-Reviewing and Alignment} 
    \label{algo:sera}
    \begin{algorithmic}[1]

    \SetKwInput{KwInput}{Input}
    \SetKwInput{KwOutput}{Output}
    
    \STATE \textbf{Input}: Preference dataset $\mathcal{D}$, SFT model $\pi_{\theta_0}$, DAA loss function $\ell$ \\
    \STATE \textbf{Hyper-parameters}: \# off-policy samples ($k$), \# bootstrapped samples ($\tilde{k}$), Training epochs $T$\\
    \STATE \textbf{Output}: Policy model $\pi_{\theta_T}$
    \STATE \textcolor{MidnightBlue}{\textbf{\texttt{/* Update policy model with offline dataset for $t=1$*/}}}
    \STATE Obtain $\pi_{\theta_{1}} \leftarrow \argmax_{\theta} \mathbb{E}_{(\vx, \vy_w, \vy_l) \in \mathcal{D}} \left[ \ell (\beta \log \frac{\pi_\theta(\vy_w|\vx)}{\pi_{\theta_{0}}(\vy_w|\vx)} - \beta \log \frac{\pi_\theta(\vy_l|\vx)}{\pi_{\theta_{0}}(\vy_l|\vx)}) \right]$
    \FOR{$t \in \{2, \ldots, T\}$}
    \STATE \textbf{Initialize} $\mathcal{M} \leftarrow \emptyset$, $\tilde{\mathcal{D}} \leftarrow \emptyset$, $\tilde{\mathcal{M}} \leftarrow \emptyset$ \\
    \FOR{($\vx^{(i)}, \vy^{(i)}_{w}, \vy^{(i)}_{l}$) $\in \mathcal{D}$}
        \STATE \textcolor{MidnightBlue}{\textbf{\texttt{/* [\S3.2] Compute IRM for Sample Selection */}}} \\
        
        \STATE $m^{(i)} \coloneqq m_{t}(\vx^{(i)}, \vy^{(i)}_{w}, \vy^{(i)}_{l})$ \hfill \textcolor{MidnightBlue}{// Margin based on ensemble reward (Eq. \ref{eqn:rw})}\\
        
        \STATE $\mathcal{M} \leftarrow \mathcal{M} \cup \{m^{(i)} \}$ \hfill \textcolor{MidnightBlue}{// Store IRM for each data-point}
        
        \STATE \textcolor{MidnightBlue}{\texttt{\textbf{/* [\S3.3] IRM-based Preference Bootstrapping */}}}
        
        \STATE $\tilde{\vy}_{1}, \ldots, \tilde{\vy}_{R} \leftarrow \pi_{\theta_{t-1}}(\cdot | \vx^{(i)})$ \hfill \textcolor{MidnightBlue}{// Sample on-policy trajectories}
        
        \STATE $\tilde{r}_{t,i,j} \leftarrow r_{t}(\vx^{(i)}, \tilde{\vy}_{j}) ~\forall j \in \{1, \dots, R\}$ \hfill \textcolor{MidnightBlue}{// Get reward for trajectories (Eq \ref{eqn:rw})}
        
        \STATE $\tilde{\mathcal{D}} \leftarrow \tilde{\mathcal{D}} \cup (\vx^{(i)}, \tilde{\vy}_{\argmax_{j} \tilde{r}_{t,i,j}}, \tilde{\vy}_{\argmin_{j} \tilde{r}_{t,i,j}})$ \hfill \textcolor{MidnightBlue}{// Select trajectories to maximize IRM}
        
        \STATE $\tilde{\mathcal{M}} \leftarrow \tilde{\mathcal{M}} \cup \{ \max_{j} \tilde{r}_{t,i,j} - \min_{j} \tilde{r}_{t,i,j}\}$ \hfill \textcolor{MidnightBlue}{// Store maximum IRM}
        \
    \ENDFOR \\
    \STATE \textcolor{MidnightBlue}{\textbf{\texttt{/* Select samples with the highest IRM */}}} \\
    
    \STATE
    $\mathcal{M}, \tilde{\mathcal{M}} \leftarrow$ sort($\mathcal{M}$, revesed=True), sort($\tilde{\mathcal{M}}$, revesed=True)\\

    \STATE
    $\mathcal{D}_t \leftarrow \{(\vx^{(i)}, \vy_w^{(i)}, \vy_l^{(i)}) \in \mathcal{D} \mid m^{(i)} \in \mathcal{M}[:k])\} \cup \{ (\vx^{(i)}, \tilde{\vy}_w^{(i)}, \tilde{\vy}_l^{(i)}) \in \tilde{\mathcal{D}} \mid \tilde{m}^{(i)} \in \tilde{\mathcal{M}}[:\tilde{k}])\}$\\

    \STATE \textcolor{MidnightBlue}{\textbf{\texttt{/* Update policy model */}}}
    \STATE Obtain $\pi_{\theta_{t}} \leftarrow \argmax_{\theta} \mathbb{E}_{(\vx, \vy_w, \vy_l) \in \mathcal{D}_t} \left[ \ell (\beta \log \frac{\pi_\theta(\vy_w|\vx)}{\pi_{\theta_{t-1}}(\vy_w|\vx)} - \beta \log \frac{\pi_\theta(\vy_l|\vx)}{\pi_{\theta_{t-1}}(\vy_l|\vx)}) \right]$
    \ENDFOR
    \end{algorithmic}
\end{algorithm}
\vspace{-5pt}

We first describe how our two proposed components can be incorporated in an iterative manner with DPO and then extend it for other DAAs. Algorithm~\ref{algo:sera} provides an overview of \alg.
As our method leverages the policy model from previous iterations both for sample selection (\S\ref{sec:reward_selection}) and preference data bootstrapping (\S\ref{sec:preference_bootstrap}), we consider an iterative approach to alignment training. Similar to recent works~\citep{kim2024sdpo, pattnaik2024curry}, we also update the reference model $\pi_{\text{ref}}$ to the latest updated policy model after every iteration. At iteration $t$, we obtain $\theta_t$ using the following update step (instead of \autoref{eqn:dpo} used in DPO):
\begin{equation}
    \theta_{t} \leftarrow \argmax_{\theta} \mathbb{E}_{(\vx, \vy_{w}, \vy_{l}) \sim \mathcal{D}_{t}} \left[ - \log \sigma \left( \beta \log \frac{\pi_{\theta}(\vy_{w}|\vx)}{\pi_{\theta_{t-1}}(\vy_{w}|\vx)} - \beta \log \frac{\pi_{\theta}(\vy_{l}|\vx)}{\pi_{\theta_{t-1}}(\vy_{l}|\vx)} \right) \right],
\end{equation}
where $\mathcal{D}_{t}$ is a dataset obtained at the start of iteration $t$, by combining $k$ samples from off-policy dataset using \S\ref{sec:reward_selection} and $\tilde{k}$ bootstrapped samples using method \S\ref{sec:preference_bootstrap} with the highest IRM.  
At $t=0$, we consider the reference model to be the SFT model (i.e. $\pi_{\theta_{0}} = \pi_{\text{SFT}}$) that we seek to align. Note that at $t=1$, $\mathcal{D}_{t}$ consists only of the original offline samples, as no aligned model is available yet.

\textbf{Ensemble of Reward Margin across Different Iteration.} 
As \alg~ depends on leveraging the policy model being trained instead of external reward model supervision, it risks introduction of undesired bias or reward hacking~\citep{pan2024spontaneous} that can manifest as part of the implicit rewards used to calculate IRM. To alleviate this, we apply $m_{t}(\vx, \vy_{w}, \vy_{l}) \coloneqq (1/\beta) \cdot \left( r_{t}(\vx, \vy_{w}) - r_{t}(\vx, \vy_{l}) \right)$ rather than using Eqn.~(\ref{eqn:rm}), where $r_{t}$ ($t \geq 3$) is defined as follows:
\begin{equation}\label{eqn:rw}
    r_{t}(\vx, \vy) = (1-\gamma) \log \frac{\pi_{\theta_{t-1}} (\vy|\vx)}{\pi_{\theta_{t-2}} (\vy|\vx)} + \gamma \log \frac{\pi_{\theta_{t-2}} (\vy|\vx)}{\pi_{\theta_{t-3}} (\vy|\vx)} = \log \frac{\pi_{\theta_{t-1}} (\vy|\vx)^{(1-\gamma)}\pi_{\theta_{t-2}} (\vy|\vx)^{\gamma}}{\pi_{\theta_{t-2}} (\vy|\vx)^{(1-\gamma)}\pi_{\theta_{t-3}} (\vy|\vx)^{\gamma}},
\end{equation}
where $\gamma$ is the ensemble coefficient. This is conducted both for preference pairs in the original datasets\,(\S\ref{sec:reward_selection}) and generated samples\,(\S\ref{sec:preference_bootstrap}) to build $\mathcal{D}_{t}$ for iteration $t$. For $t < 3$, we use $r_{t}(\vx, \vy) = \log \frac{\pi_{\theta_{t-1}} (\vy|\vx)}{\pi_{\theta_{t-2}} (\vy|\vx)}$. Empirically, this approach showed consistent effectiveness (\S\ref{sec:analysis}; \autoref{fig:sub1}).

\textbf{Extension to Other DAAs.} Instead of BCE loss function of IRM~(\textit{i.e.,} $m(\vx, \vy_{w}, \vy_{l})$) in Eqn.\,(\ref{eqn:dpo}), SLiC-HF~\citep{zhao2023slic} minimizes a hinge loss function of IRM:%
\begin{equation}
    \mathcal{L}_{\text{SLiC-HF}}(\vx, \vy_{w}, \vy_{l}) = \max \left( 0, 1 - \beta \left( {\color{Orchid!60!black} \log \frac{\pi_{\theta}(\vy_{w}|\vx)}{\pi_{\text{ref}}(\vy_{w}|\vx)} - \log \frac{\pi_{\theta}(\vy_{l}|\vx)}{\pi_{\text{ref}}(\vy_{l}|\vx)} } \right) \right),
\end{equation}%
where $1/\beta$ acts as the margin of miss-classification~\citep{guo2024direct}. IPO~\citep{azar2024general} minimizes the square loss function of IRM:
\begin{equation}
    \mathcal{L}_{\text{IPO}}(\vx, \vy_{w}, \vy_{l}) = \left( \left( {\color{Orchid!60!black} \log \frac{\pi_{\theta}(\vy_{w}|\vx)}{\pi_{\text{ref}}(\vy_{w}|\vx)} - \log \frac{\pi_{\theta}(\vy_{l}|\vx)}{\pi_{\text{ref}}(\vy_{l}|\vx)} } \right) - \frac{1}{2\beta} \right)^{2}
\end{equation}
Note that IPO and SLiC-HF, in contrast to DPO, doesn’t assume a preference model like BT.  Despite of such independence, the objective functions still contain the {\color{Orchid!60!black} IRM term} enabling us to easily extend \alg~to these DAAs. Thus, we use \alg~with these DAAs for our upcoming experiments.

\section{Experiments}\label{sec:exp}
\subsection{Experimental Setup}\label{sec:setup}
\paragraph{Setup.} 
Our setup resembles \citet{tunstall2023zephyr} and \citet{hong2024reference}, where they consider three models-- TinyLlama\,\citep{zhang2024tinyllama}, Pythia-2.8B\,\citep{biderman2023pythia}, and Mistral-7B\,\citep{jiang2023mistral}. Initially, these models are fine-tuned on UltraChat-200k~\citep{tunstall2023zephyr} followed by an alignment with DAAs and preference pairs from UltraFeedback. We use the binary version of UltraFeedback (which contains two response pairs with corresponding ratings for a given input query)~\citep{tunstall2023zephyr} for all DAAs except Curry-DPO~\citep{pattnaik2024curry}, where we use the original UltraFeedback (that contains four response pairs)~\citep{cui2023ultrafeedback}. In addition, we will also showcase \alg's prowess on other popular alignment datasets like HH-RLHF~\citep{bai2022training} and TL;DR~\citep{stiennon2020learning}. A detailed description of the datasets and the implementation can be found in \autoref{app:setup}.  For all experiments, we set $T=3$, $\gamma=0.3$. If $N$ represent the total number of preference pairs in the offline dataset, we choose our values of $k, \tilde{k} \text{ s.t. }  k + \tilde{k} = N$, allowing us a fair comparison against other baseline in terms of sample size. Based on a validation set, we set $k=0.7N\,,\,\tilde{k}=0.3N$ (labelled as \alg~ in our tables). However, we find our proposed method performs best when we relax the sum-constraint and consider a higher proportion of generated data (see analysis in \S\ref{k_and_k_dash}) and report these numbers for \texttt{\textbf{best}}-\alg, where  $k=0.7N\,,\, \tilde{k}=2N$.

\paragraph{Baselines.} We compare \alg~to several previous DAAs. First, we consider DPO\,\citep{rafailov2024direct} and its variants: Iterative DPO\,\citep{kim2024sdpo, rosset2024direct}, Curry-DPO\,\citep{pattnaik2024curry}, SPIN\,\citep{chen2024self}, and ORPO\,\citep{hong2024reference}. Then, we use \alg~with IPO\,\citep{azar2024general} and SLiC-HF\,\citep{zhao2023slic} and compare against the original method and its iterative version. For experiments on synthetic noisy preference datasets, we also include cDPO\,\citep{mitchell2023note} and rDPO\,\citep{chowdhury2024provably}.

\begin{table}[t]
    \centering
    \caption{\small Comparison of performance where TinyLlama-1.1B, Pythia-2.8B, and Mistral-7B are fine-tuned on UltraChat-200k and preference pairs are sampled from UltraFeedback dataset with DPO~\citep{rafailov2024direct} or its variants~\citep{pattnaik2024curry, chen2024self, hong2024reference}. Single represents average from single answer grading prompt while SFT, G3.5, G4 indicate the pairwise comparison~\citep{zheng2024judging} between SFT model, \texttt{text-davinci-003}, and \texttt{gpt4\_turbo}. The best and second best performances are highlighted \textbf{bold} and \underline{underline}. $\ddagger$ indicate results from original models shared on HuggingFace.}
    \vspace{-5pt}
    \resizebox{0.9\columnwidth}{!}{
    \begin{tabular}{c|c|c|cccc|cc|cc|cc}
        \toprule[0.1em]
        \multirow{2}{*}{\textbf{Model}} & \multirow{2}{*}{\textbf{Size}} & \multirow{2}{*}{\textbf{Technique}} & \multicolumn{4}{c|}{\textbf{Alpaca Eval}} & \multicolumn{2}{c|}{\textbf{Vicuna Eval}} & \multicolumn{2}{c|}{\textbf{Evol-Instruct}} & \multicolumn{2}{c}{\textbf{UltraFeed}} \\
         & & & Single & SFT & G3.5 & G4 & Single & SFT & Single & SFT & Single & SFT \\

        \midrule[0.1em]
        \multirow{6}{*}{TinyLlama} & \multirow{6}{*}{1.1B} & SFT & 5.41 & - & 28.9 & 1.7 & 6.05 & - & 5.48 & - & 4.98 & - \\
        & & DPO & 5.82 & 55.1 & 33.5 & 2.9 & 6.31 & 57.5 & 5.50 & 56.0 & 5.34 & 59.8 \\
        & & Iterative DPO & 6.08 & 64.5 & 36.1 & {3.2} & 7.09 & 77.5 & 6.03 & 62.4 & 5.49 & 64.4 \\
        & & Curry-DPO & 6.16 & 62.3 & 40.7 & 2.8 & 6.85 & 77.5 & 6.00 & 61.5 & 5.63 & 66.8 \\
        & & \texttt{\textbf{SeRA}}-DPO & \textbf{6.58} & \textbf{74.2} & \underline{48.7} & \underline{5.1} & \textbf{7.55} & \textbf{85.8} & \textbf{6.40} & \textbf{72.3} & \underline{5.92} & \underline{73.1} \\
        & & \texttt{\textbf{best-SeRA}}-DPO & \underline{6.47} & \underline{70.0} & \textbf{49.1} & \textbf{6.1} & \underline{7.31} & \underline{82.5} & \underline{6.18} & \underline{70.0} & \textbf{6.06} & \textbf{79.5} \\
        \midrule
        
        \multirow{6}{*}{Pythia} & \multirow{6}{*}{2.8B} & SFT & 5.34 & - & 28.1 & 2.2 & 6.19 & - & 5.73 & - & 5.09 & - \\
        &  & DPO & 5.67 & 53.0 & 31.6 & 2.9 & 6.30 & 50.0 & 6.02 & 58.7 & 5.38 & 57.5 \\
        &  & Iterative DPO & 6.05 & 63.2 & 40.6 & {3.0} & {7.23} & {70.0} & 6.16 & 60.8 & {5.67} & 65.9 \\
        & & Curry-DPO & {6.06} & 62.7 & 39.6 & 2.7 & 7.03 & 62.5 & 6.39 & {68.1} & 5.67 & \underline{72.0} \\
        &  & \texttt{\textbf{SeRA}}-DPO & \underline{6.40} & \underline{70.7} & \underline{50.2} & \underline{6.0} & \textbf{7.59} & \underline{76.8} & \textbf{6.68} & \underline{74.2} & \underline{5.93} & {71.7} \\
        &  & \texttt{\textbf{best-SeRA}}-DPO & \textbf{6.59} & \textbf{77.4} & \textbf{52.2} & \textbf{7.7} & \underline{7.53} & \textbf{83.8} & \underline{6.54} & \textbf{79.1} & \textbf{6.00} & \textbf{76.6} \\
        \midrule
        
        \multirow{8}{*}{Mistral} & \multirow{8}{*}{7B} & SFT$^{\ddagger}$ & 7.46 & - & 75.9 & 5.4 & 8.08 & - & 7.73 & - & 6.76 & - \\
        &  & DPO$^{\ddagger}$ & 7.37 & 71.4 & 74.9 & 17.5 & 8.39 & 83.8 & 7.15 & 61.7 & 6.42 & 65.6 \\
        &  & Iterative DPO & 8.35 & \underline{89.0} & 93.2 & \underline{19.7} & 8.74 & 88.2 & 8.49 & \underline{87.8} & 7.56 & 81.8 \\
        &  & Curry-DPO & 8.28 & 87.6 & 93.1 & 13.8 & 8.75 & \textbf{92.5} & 8.45 & 85.8 & 7.52 & 80.8 \\
        & & SPIN$^{\ddagger}$ & 7.64 & 65.1 & 77.1 & 8.4 & 8.40 & 70.0 & 7.66 & 61.7 & 6.63 & 61.0 \\
        & & ORPO$^{\ddagger}$ & 8.29 & 84.1 & 91.6 & 14.5 & 8.82 & 83.8 & 8.61 & 86.0 & 7.58 & 78.8 \\
        &  & \texttt{\textbf{SeRA}}-DPO & \underline{8.38} & 86.9 & \underline{93.4} & 19.5 & \underline{8.83} & 88.8 & \underline{8.61} & 86.0 & \underline{7.67} & \underline{84.4} \\
        &  & \texttt{\textbf{best-SeRA}}-DPO & \textbf{8.56} & \textbf{92.7} & \textbf{94.8} & \textbf{27.3} & \textbf{8.94} & \underline{91.3} & \textbf{8.68} & \textbf{90.1} & \textbf{7.90} & \textbf{90.6} \\
        \bottomrule[0.1em]
    \end{tabular}
    }
    \label{tab:main}
\afterfigspace
\end{table}

\paragraph{Evaluation.} We consider four popular test benchmarks -- Alpaca bench~\citep{dubois2024alpacafarm}, Vicuna bench~\citep{chiang2023vicuna}, Evol Instruct~\citep{xu2023wizardlm} test set, and the UltraFeedback~\citep{cui2023ultrafeedback, tunstall2023zephyr} test set. For metrics, we showcase absolute metrics that use a `Single' answer-grading prompt using Claude 3~\citep{claude3} as a judge to evaluate the quality of the generated response. In addition, we also report win-rates against generations from other models like an SFT-model, GPT-3.5 (G3.5), and GPT-4 (G4). The win-rate adjudication assigns a weight of $1$ to wins, $0.5$ to ties, and $0$ to a loss. We also use a single answer grading prompt to provide a 1-10 rating for the responses generated by the model.  The evaluation prompts are taken from \cite{zheng2024judging} and provided in Appendix~\ref{app:evaluation} for reference.

\subsection{Results}\label{sec:result}

\paragraph{DPO-baselines} We assess the general instruction-following abilities of LLMs by comparing DPO-based preference alignment algorithms in \autoref{tab:main}.
We observe that (1) the instruction-following performance of both variants of \alg outperforms baselines on all (single and pair-wise) metrics, across diverse LLMs, and evaluation metrics, and (2) increased preference bootstrapping is more effective for LLMs with higher model capacity (\textit{e.g.,} Mistral-7B), as seen in the larger performance gap between variants.
We believe that these results highlight the generality of the proposed method across different experimental setups. We now seek to showcase that \alg~is algorithm-agnostic and works well across various DAAs.

\begin{table}[t!]
    \centering
    \caption{\small Comparison of performance where TinyLlama-1.1B and Pythia-2.8B are fine-tuned on UltraChat-200k and preference pairs are sampled from UltraFeedback dataset with IPO~\citep{azar2024general} and SLiC-HF~\citep{zhao2023slic}. The best performances are highlighted in \textbf{bold}.}
    \vspace{-5pt}
    \resizebox{0.9\columnwidth}{!}{
    \begin{tabular}{c|c|c|cccc|cc|cc|cc}
        \toprule[0.1em]
        \multirow{2}{*}{\textbf{Method}} & \multirow{2}{*}{\textbf{Size}} & \multirow{2}{*}{\textbf{Technique}} & \multicolumn{4}{c|}{\textbf{Alpaca Eval}} & \multicolumn{2}{c|}{\textbf{Vicuna Eval}} & \multicolumn{2}{c|}{\textbf{Evol-Instruct}} & \multicolumn{2}{c}{\textbf{UltraFeed}} \\
         & & & Single & SFT & G3.5 & G4 & Single & SFT & Single & SFT & Single & SFT \\

        \midrule[0.1em]
        \multirow{7}{*}{TinyLlama} & \multirow{7}{*}{1.1B} & SFT & 5.41 & - & 28.9 & 1.7 & 6.05 & - & 5.48 & - & 4.98 & - \\
        \cmidrule{3-13}
        & & IPO & 5.94 & 54.1 & 29.3 & 2.7 & 6.84 & 63.8 & 5.91 & 50.7 & 5.55 & 64.7 \\
        & & Iterative IPO & 5.98 & 56.5 & 30.1 & 2.2 & 6.94 & 63.8 & 6.00 & 59.2 & 5.65 & 64.7 \\
        & & \texttt{\textbf{SeRA}}-IPO & \textbf{6.42} & \textbf{69.3} & \textbf{45.6} & \textbf{4.8} & \textbf{7.31} & \textbf{81.3} & \textbf{6.12} & \textbf{69.2} & \textbf{5.76} & \textbf{70.5} \\
        \cmidrule{3-13}

        & & SLiC & 6.19 & 65.9 & 40.1 & 3.5 & 7.27 & 71.3 & 6.05 & 65.1 & 5.57 & 66.8 \\
        & & Iterative SLiC & 6.41 & 70.6 & \textbf{47.8} & 4.8 & 7.23 & \textbf{78.8} & 6.06 & 65.1 & 5.73 & 69.1 \\
        & & \texttt{\textbf{SeRA}}-SLiC & \textbf{6.50} & \textbf{70.8} & 47.2 & \textbf{5.4} & \textbf{7.58} & 76.3 & \textbf{6.06} & \textbf{67.9} & \textbf{5.88} & \textbf{72.8} \\
        \midrule
        
        \multirow{7}{*}{Pythia} & \multirow{7}{*}{2.8B} & SFT & 5.34 & - & 28.1 & 2.2 & 6.19 & - & 5.73 & - & 5.09 & - \\
        \cmidrule{3-13}
        
        &  & IPO & 6.15 & 67.1 & 41.8 & 4.5 & 7.32 & 80.0 & 6.27 & 66.5 & 5.65 & 75.3 \\
        &  & Iterative IPO & 6.55 & 75.2 & 50.2 & 5.0 & 7.61 & 80.0 & 6.68 & 71.3 & 6.02 & 75.5 \\
        &  & \texttt{\textbf{SeRA}}-IPO & \textbf{6.71} & \textbf{78.5} & \textbf{53.6} & \textbf{6.1} & \textbf{7.74} & 83.8 & \textbf{6.83} & \textbf{72.7} & \textbf{6.14} & \textbf{78.7} \\
        \cmidrule{3-13}

        &  & SLiC & 6.07 & 65.2 & 41.7 & 2.8 & 7.06 & 66.3 & 6.24 & 60.8 & 5.66 & 69.9 \\
        &  & Iterative SLiC & 6.38 & 72.3 & 47.1 & 5.6 & 7.54 & 83.8 & \textbf{6.48} & \textbf{73.4} & 5.84 & 70.0 \\
        &  & \texttt{\textbf{SeRA}}-SLiC & \textbf{6.57} & \textbf{74.3} & \textbf{52.9} & \textbf{6.2} & \textbf{7.62} & \textbf{85.0} & 6.47 & 71.1 & \textbf{6.03} & \textbf{74.3} \\
        \bottomrule[0.1em]
    \end{tabular}
    }
    \label{tab:optim}
\afterfigspace \afterfigspace \afterfigspace \afterfigspace
\end{table}

\paragraph{DAA-agnoistic \alg} To show that \alg~can be improve various DAAs, we report instruction-following performance of \alg~when integrated with IPO~\citep{azar2024general} and SLiC-HF~\citep{zhao2023slic} in \autoref{tab:optim}. For Iterative IPO and Iterative SLiC, we replace the reference model at every iteration with the model from the previous iteration, unlike vanilla IPO and SLiC-HF, which maintain the reference model as the SFT model. Similar to \autoref{tab:main}, \alg~consistently improves the performance of both IPO and SLiC-HF. Compared to the DPO and SLiC-HF cases, the performance improvement on when \alg~is used with IPO is relatively smaller. We note that IPO has been shown to effectively alleviate over-optimization in comparison to other DAAs~\citep{rafailov2024scaling}; hence the gains obtained by \alg~indicate other potential upsides. Indirectly, these results support our argument that \alg~addresses the challenge of over-optimization. We additionally provided the effectiveness of \alg~on different types of implicit reward, SimPO~\citep{meng2024simpo}, in Appendix~\ref{app:simpo}. 

\paragraph{Robustness to Noisy Preferences.}
\begin{wrapfigure}{r}{0.68\textwidth}
\begin{center}
    \includegraphics[width=1.0\linewidth]{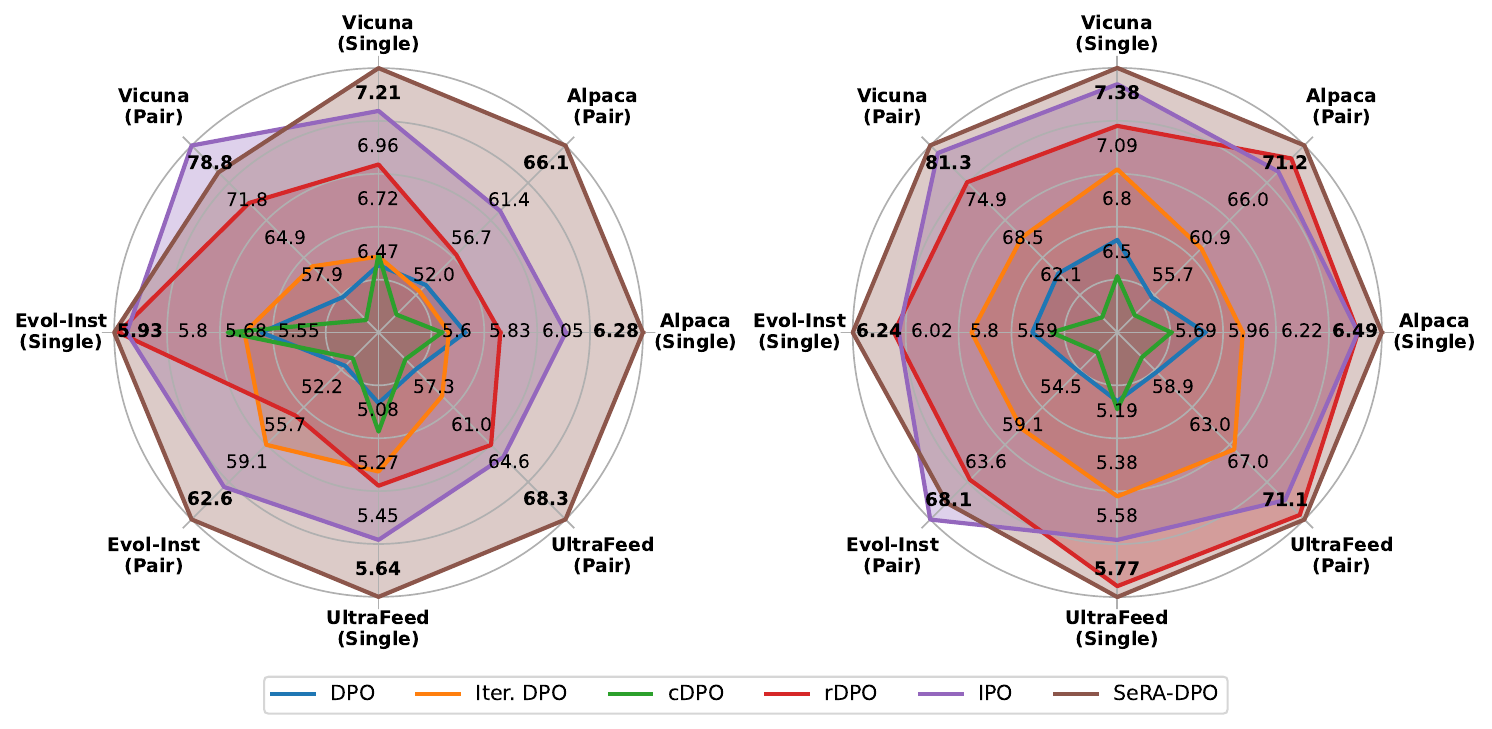}
    \caption{
    Comparison of performance on 20\%\,\textbf{(left)} and 40\%\,\textbf{(right)} of noisy preference
    }
    \label{fig:noisy}
\end{center}
\end{wrapfigure}

To illustrate the effectiveness of our proposed \alg~on noisy annotations, we experimented with synthetically noised preference datasets. Similar to prior works \citep{chowdhury2024provably, wang2024secrets}, we consider the standard noise model and randomly swap the chosen and rejected responses. Our experiments evaluate two setups-- 20\% and 40\% noise in the UltraFeedback dataset.
On TinyLlama trained on each noise split, we highlight the performance of various robust DAAs in \autoref{fig:noisy}. In most of the evaluation benchmarks, \alg~ showed superiority over other baselines. Moreover, the performance gap between DPO, iterative DPO, and \texttt{\textbf{SeRA}}-DPO is much larger than that on standard dataset. This indicate that our reward margin selection effectively filters-out noisy samples with swapped preferences.

\paragraph{Efficacy on Different Preference Datasets.}
\begin{wraptable}{r}{0.5\textwidth}
\caption{
    Results on HH-RLHF and TL;DR.
}\label{tab:tldr}
\vspace{-5pt}
\resizebox{0.5\textwidth}{!}{%
\setlength{\tabcolsep}{2pt}{
\begin{tabular}{l|ccc|cc}
\toprule
        & DPO    & IDPO & OAIF & \makecell{\alg\\($\Tilde{k}=0$)} & \makecell{\alg\\($\Tilde{k}=0.3N$)} \\ \midrule
HH-RLHF & 56.7 & 63.0    & 64.6 & \underline{65.4} & \textbf{66.1}   \\
TL;DR   & 59.7 & 59.6    & 59.5 & \underline{61.0} & \textbf{62.7}   \\ \bottomrule
\end{tabular}}}
\afterfigspace
\end{wraptable}

To demonstrate the generality of \alg, we also conducted experiments with TinyLlama on two other datasets: HH-RLHF~\citep{ganguli2022red} and TL;DR~\citep{stiennon2020learning}. Since these two datasets do not contain explicit GPT-4 scores for each response, we utilized a Mistral-7B-based reward model\footnote{\texttt{https://huggingface.co/weqweasdas/RM-Mistral-7B}; this model ranked 6th in April 2024 and achieved the highest performance among Mistral-based models in September 2024 in RewardBench.} that achieves high performance in RewardBench~\citep{lambert2024rewardbench} to implement online feedback from stronger LLMs (OAIF; \citealt{guo2024direct}), which requires significantly more computational resources compared to ours. Further, we also showcase the importance of sample selection alone by benchmarking \alg~with $\tilde{k}=0$. \autoref{tab:tldr} shows that the effectiveness of 
our approach is not limited to any particular preference dataset. It overcomes problems like spurious correlations that practically exists in every preference dataset. Further, it can achieve this better performance at a lower cost.

\section{Analysis and Discussion}\label{sec:analysis}

\subsection{Effects of Off-policy sampling ($k$) and preference bootstrapping ($\tilde{k}$)}
\label{k_and_k_dash}

\begin{figure*}[t]
\centering
\begin{subfigure}[b]{0.4\textwidth}
    \begin{center}
    \includegraphics[width=1.0\linewidth]{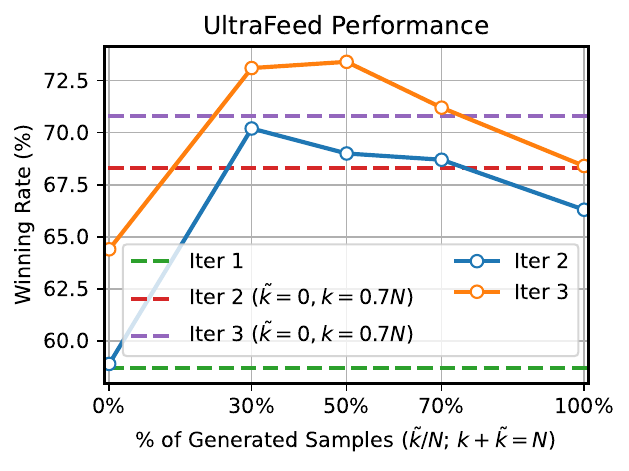}
    \caption{Pairwise comparison of models with varying percentage $k$ and $\tilde{k}$ while ensuring a fixed-sized training set (i.e. $\tilde{k}; k + \tilde{k} = N$).}
    \label{fig:comp}
    \end{center}
\end{subfigure}
\qquad
\begin{subfigure}[b]{0.48\textwidth}
\begin{center}
    \includegraphics[width=1.0\linewidth]{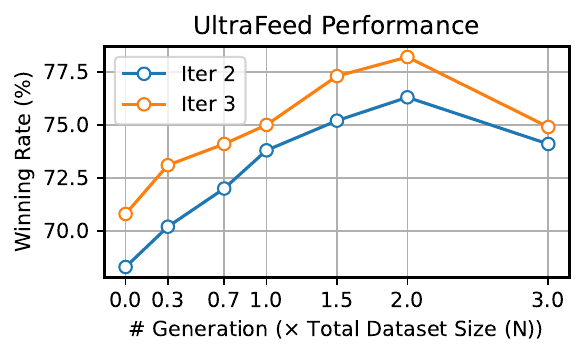}
    \caption{Pairwise comparison of models to study the effect of using varing amount of generated data ($\tilde{k}$) while using a fixed amount of off-policy data $k=0.7N$.}
    \label{fig:gen}
\end{center}
\end{subfigure}
\caption{Analyzing the effect of varying percentage (determined by $k$) of sampled off-policy preference data, and generated data (determined by $\tilde{k}$).}
\end{figure*}

We evaluate the efficacy of sample selection and preference bootstrapping in our proposed method. We plot the winning rate for different ratios of $k$ and $\tilde{k}$ in \autoref{fig:comp}. The solid lines represent results obtained by different iterations in \alg~ with varying ratios of $k$ and $\tilde{k}$ maintaining $k + \tilde{k} = N$, while the dotted lines represent results obtained by using only IRM selection where $k=0.7N, \tilde{k}=0$. The results indicate that (1) IRM-based selection shows consistent improvement in comparison to using only offline samples ($0$\% on x-axis), or only generated samples ($100$\%), and (2) combining IRM-based sample selection and preference bootstrapping at moderate levels (\textit{i.e.,} $30-70$\%) leads to consistent improvement, highlighting the importance of both components in \alg. Without offline samples, IRM and generated samples may lack valuable information for effective training, while relying solely on offline data can cause overfitting. Thus, the combination of selection and bootstrapping is crucial. Across iterations, this combination consistently outperforms solely using all offline samples, only generated samples, or IRM selection alone, confirming the importance and synergy of each component in \alg.

To uncover the full potential of \alg, we increase the number of self-generated samples in training the policy model. We fix $k$ at $0.7N$. We compare the cases with $\tilde{k}$ ranging from $0$ to $3N$. The results in \autoref{fig:gen} illustrate that increasing the generated samples up to $2N$ leads to performance improvement, highlighting the efficacy of preference bootstrapping. However, we observe a minor drop at $3N$, possibly pointing to undesirable model bias and less-diverse trajectories generated by the model. We hope to consider advanced decoding approaches \cite{huang2024deal} to address the latter in the future.

\subsection{Analysis on Ensemble across Different Iterations}

We also investigate the efficacy of the ensemble across different iterations, reporting the performance and similarity of selected samples for the TinyLlama case, when combining IRM in the second and third iterations. Although we see high winning rate in \autoref{fig:sub1}, a moderate range of ensemble coefficients $\gamma$ (\textit{i.e.,} 0.1–0.5) leads to greater performance improvements compared to higher coefficients (\textit{i.e.,} 0.7–0.9). This suggests that the reward margin generated by the latter policy models play a more significant role. However, employing an ensemble of reward margins can effectively mitigate the deterioration caused by undesirable model bias. Additionally, \autoref{fig:sub2} examines the similarity of selected samples across different $\gamma$ values. The results show that mild levels of $\gamma$ only slightly alter the selected samples, with similarity scores ranging from 0.88 to 0.94. This slight modification is crucial as it ensures the retention of important samples while preventing undesired bias in the model.

\begin{figure*}[t]
\centering
\begin{minipage}[b]{0.7\textwidth}
\centering
\hfill
\begin{subfigure}[b]{0.51\textwidth}
    \centering
    \includegraphics[width=\textwidth]{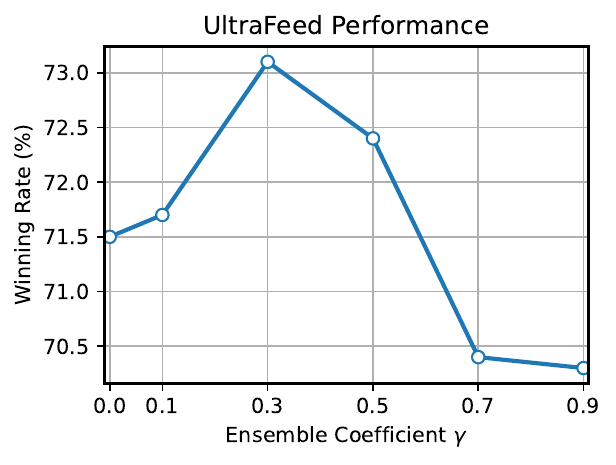}
    \caption{Final Performance}
    \label{fig:sub1}
\end{subfigure}
\hfill
\begin{subfigure}[b]{0.46\textwidth}
    \centering
    \includegraphics[width=\textwidth]{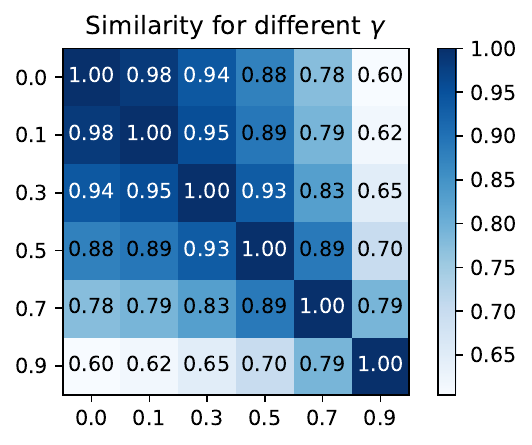}
    \caption{Similarity}
    \label{fig:sub2}
\end{subfigure}
\hfill
\caption{
    The performance and similarity of selected sample sets across different values of the ensemble coefficient $\gamma$ in TinyLlama-1.1B.
}
\end{minipage}
\label{fig:ens}

\begin{minipage}[b]{0.7\textwidth}
\centering
\hfill
\begin{subfigure}[b]{0.47\textwidth}
    \centering
    \includegraphics[width=\textwidth]{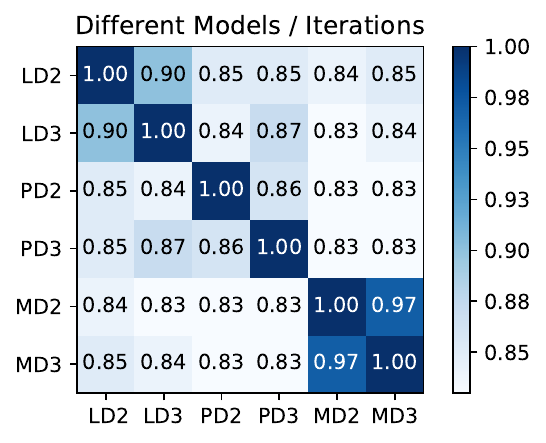}
    \caption{LLMs}
    \label{fig:sub3}
\end{subfigure}
\hfill
\begin{subfigure}[b]{0.47\textwidth}
    \centering
    \includegraphics[width=\textwidth]{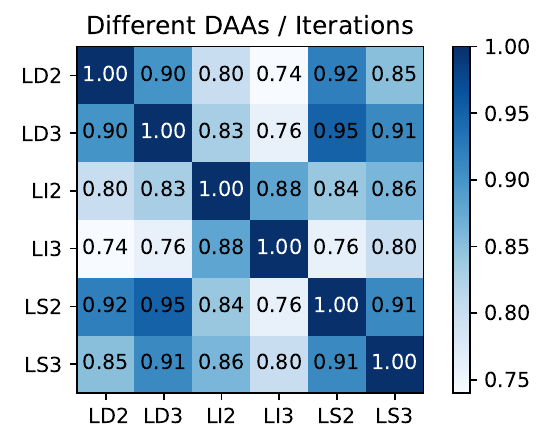}
    \caption{DAAs}
    \label{fig:sub4}
\end{subfigure}
\hfill
\caption{The similarity of selected sample sets between different LLMs or DAAs with training iterations.}
\end{minipage}
\vspace{-12.5pt}
\end{figure*}

\subsection{Selection Behavior}

We checked the difference in the selected sample set with different training configurations such as DAAs, iterations, and LLMs. \autoref{fig:sub3} and \autoref{fig:sub4} depict the Jaccard similarity\,\citep{murphy1996finley} between two different selected sample sets from different LLMs and DAAs, respectively. All selected sets contain 70\% of the samples from the original UltraFeedback dataset. For all axes, the first and second letters indicate the LLM (e.g., L, P, and M indicate TinyLlama, Pythia-2.8B, and Mistral-7B) and DAA (e.g., D, I, and S indicate DPO, IPO, and SLiC-HF), and the third number indicates the iteration count. In general, we observe high similarity in the selected samples across LLMs ($\geq 0.83$) and across DAAs ($\geq 0.74$). Across iterations, we also observe a slight variance in the samples selected ($0.903$ on avg.).

\section{Conclusion}
In this work, we proposed \alg that incorporates two components into Direct Alignment Algorithms (DAAs)-- (1) an implicit reward margin-based sample selection of offline datasets mitigates over-optimization, supported by empirical evidence; and (2) a self-reviewed preference bootstrapping to enable direct feedback that considers the training-inference mismatch. In addition, we use a reward ensemble over the updated policy models across iterations; this helps to develop a robust reward metric and in-turn a higher-quality on-policy dataset (although this increases the memory costs). Extensive experiments on diverse training configurations, across various LLMs, DAAs, and datasets, demonstrated the superior performance of \alg. This highlights that even without additional supervision from humans or stronger LLMs, we can (1) reap further benefits from existing offline preference datasets and (2) obtain labels for on-policy data using the policy model's metrics (e.g. reward margins) to improve the policy model. Similar to ~\citet{burns2023weak}, our work demonstrates the possibility for self-evolution of LLMs to achieve better alignment with weak or minimal external supervision.

\bibliography{main}
\bibliographystyle{iclr2024_conference}

\appendix
\clearpage
{
    \newpage
    \centering
    \Large
    \textbf{Supplementary Material} \\
    \vspace{0.5em}
    \large
    \textit{\alg: Self-Reviewing and Alignment of LLMs using Implicit Reward Margins} \\
}

\section{Proof for Theoretical Intuition}\label{app:proof}
Here, we provide the derivation of our theoretical results described in \autoref{thm:main} that deepens the mathematical understanding on how our reward margin based sample selection can prevent the over-optimization.

Consider a sample $s$ defined as $\mathbf{s} = (\vx, \vy_{w}, \vy_{l})$ where $\vx$ is a given prompt/input, and $(\vy_{w} \vy_{l})$ are two outputs where one is preferred over the other $\vy_{w} \succ \vy_{l}$. We can now define a margin as $f \coloneqq r(\vx, \vy_w) - r(\vx, \vy_l)$ that encodes the difference is reward $r$ assigned to the two output trajectories. We can now define a risk function $R(f)$ over this margin by leveraging a loss function $\ell$ over $f$ as follows:
\begin{align}\label{eq:risk}
    R(f) &\coloneqq \mathbb{E}_{\mathbf{s} \sim \mathbb{P}} \left[ \ell(f(\mathbf{s})) \right] \coloneqq \mathbb{E}_{\mathbf{s}} \left[  p^{*}(\mathbf{s})^\intercal \ell(f(\mathbf{s})) \right]
\end{align}

As the true preference distribution $\mathbb{P}$ is unknown, we can approximate \autoref{eq:risk} via the empirical risk over sample a $N$-sized set of samples $S \sim \mathbb{P}^{N}$ as,
\begin{align}\label{eq:empirical}
    \hat{R}(f; S) &\coloneqq \frac{1}{N} \sum_{n=1}^{N} \ell(f(\mathbf{s}_n)) 
\end{align}

We can also represent \autoref{eq:risk} as the Bayes-distilled risk over the sample set $S \sim \mathbb{P}^{N}$ as,
\begin{equation}\label{eq:bayes}
    \hat{R}_{*}(f; S) \coloneqq \frac{1}{N} \sum_{n=1}^{N} p^{*}(\mathbf{s}_n)^\intercal \ell(f(\mathbf{s}_n))
\end{equation}

Now, both the standard empirical risk $\hat{R}(f; S)$ (in \autoref{eq:empirical}) and Bayes-distilled risk $\hat{R}_{*}(f; S)$ (in \autoref{eq:bayes}) are unbiased estimates or the population risk $R(f)$. First, we show that the Bayes-distilled risk has lower variance, compared to its empirical risk counterpart.

\begin{lemma}[\citealt{pmlr-v139-menon21a}]\label{thm:lemma3}
    For any fixed predictor $f: \mathcal{S} \rightarrow \mathbb{R}$,
    \begin{equation*}
        \mathbb{V}_{S \sim \mathbb{P}^{N}} \left[ \hat{R}_{*}(f; S) \right] \leq \mathbb{V}_{S \sim \mathbb{P}^{N}} \left[ \hat{R}(f; S) \right],
    \end{equation*}
    where $\mathbb{V}$ denotes variance, and equality holds if and only if $\forall \mathbf{s} \in \mathcal{S}$, the loss values $\ell(f(\mathbf{s}))$ are constant on the support of $p^{*}(\mathbf{s})$.
\end{lemma}
\begin{proof}
    The detailed proof is illustrated in Lemma 1 of \citet{pmlr-v139-menon21a}.
\end{proof}

Given Lemma~\ref{thm:lemma3}, we note that the Bayes-distilled risk is more effective estimator that the standard empirical risk due to its small variance. However, since the value of \( p^*(x) \) is unknown, we can approximate \( p^*(x) \) with \( \hat{p}(x) \), which is the estimated probability produced by the model. Additionally, we approximate \autoref{eq:bayes} using the distilled risk over a sample \( S \sim \mathbb{P}^{N} \) as,
\begin{equation*}
    \tilde{R}(f; S) \coloneqq \frac{1}{N} \sum_{n=1}^{N} \hat{p}(\mathbf{s}_n)^\intercal \ell(f(\mathbf{s}_n))
\end{equation*}

Now, we provide the mathematical intuition that IRM-based sample selection can be effective by suggesting the upper bound of the distilled risk, computed by $\hat{p}(\cdot)$. For this, we first definte two terms.

\begin{definition}[Covering Number]
    Let $(A, d)$ be a metric space where $A$ is a set of points and $d$ is a distance measure. A set $C$ is an $\epsilon$-cover of $A$ if $\forall x\in A,~\exists y \in C$ such that $d(x, y) < \epsilon$. The covering number $\mathcal{N}(\epsilon, A, d)$ is defined as size of the smallest $\epsilon$-cover:
    \begin{equation*}
        \mathcal{N}(\epsilon, A, d) = \min \{ |C| \text{\; s.t.\; } C \text{ \;is an\; } \epsilon\text{-cover} \}
    \end{equation*}
\end{definition}

\begin{definition}[Uniform Covering Number]
    For $\epsilon > 0$, a function class $\mathcal{H}$ and an integer $N$, the uniform covering number, $\mathcal{N}_{\infty}(\epsilon, \mathcal{H}, N)$ is defined as
    \begin{equation*}
        \mathcal{N}_{\infty}(\epsilon, \mathcal{H}, N) = \sup_{\vx \in X^{N}} \mathcal{N}(\epsilon, \mathcal{H}(x), \| \cdot \|_{\infty}),
    \end{equation*}
    where $\mathcal{H}(x) = \{ ( f(x_1), \ldots f(x_N)) | f \in \mathcal{H} \}$ and for $A \subseteq \mathbb{R}^{N}$ the number $\mathcal{N}(\epsilon, A, \|\cdot\|_{\infty})$ is the smallest cardinality $|A_0|$ of a set $A_0 \subseteq A$ such that $A$ is contained in the union of $\epsilon$-balls centered at points in $A_0$ in the metric induced by $\| \cdot \|_{\infty}$.
\end{definition}

Here, we also describe in Lemma~\ref{thm:lemma1} and Lemma~\ref{thm:lemma2} based on above definition.
\begin{lemma}[Theorem 6, \citealt{maurer2009empirical}]\label{thm:lemma1}
    Let $X$ be a random variable with values in a set with $\mathcal{X}$ with distribution $\mathbb{P}$ and let $\mathcal{F}$ be a class of hypotheses $f: \mathcal{X} \rightarrow [0, 1]$. Fix $\delta \in (0, 1)$, $N \geq 16$ and $\mathcal{M}_{N} = \mathcal{N}_{\infty}(\frac{1}{N}, \mathcal{H}, 2N)$. Then, with probability at least $1 - \delta$ over $S \sim \mathbb{P}^{N}$, we have
    \begin{equation*}
        \mathbb{P}(f) - \mathbb{P}_{N}(f; S) \leq \sqrt{18 \tilde{\mathbb{V}}_{N} \cdot \frac{\log(10\mathcal{M}_{N}/\delta)}{N}} + \frac{15 \log (10\mathcal{M}_{N}/\delta)}{N-1}
    \end{equation*}
\end{lemma}
\begin{proof}
    The detailed proof is illustrated in \citet{maurer2009empirical}.
\end{proof}

\begin{lemma}[Modification from \citealt{pmlr-v139-menon21a}]\label{thm:lemma2}
    Pick any bounded loss $\ell$. Fix a hypothesis class $\mathcal{F}$ of predictors $f: \mathcal{X} \rightarrow \mathbb{R}^{L}$, with induced class $\mathcal{H}^{*} \subset [0, 1]^{\mathcal{X}}$ of function $h(\mathbf{s}) \coloneqq p^{*}(\mathbf{s}) \ell (f(\mathbf{s}))$. Suppose $\mathcal{H}^{*}$ has uniform covering number $\mathcal{N}_{\infty}$. Then, for any $\delta \in (0, 1)$, with probability at least $1-\delta$ over $S \sim \mathbb{P}^{N}$,
    \begin{equation*}
        R(f) \leq \hat{R}_{*}(f; S) + \mathcal{O} \left( \sqrt{\tilde{\mathbb{V}}_{N} \cdot \frac{\log(\mathcal{M}_{N}/\delta)}{N}} + \frac{\log(\mathcal{M}_{N}/\delta)}{N} \right),
    \end{equation*}
    where $\mathcal{M}^{*}_{N} \coloneqq \mathcal{N}_{\infty}(\frac{1}{N}, \mathcal{H}^{*}, 2N)$ and $\mathbb{V}^{*}_{N}(f)$ is the empirical variance of the loss values $\left\{ p^{*}(x_n)^\intercal \ell(f(x_n)) \right\}_{n=1}^{N}$.
\end{lemma}
\begin{proof}
    Note that the use of Big-O ($\mathcal{O}$) lets us drop the constants and consider $\frac{1}{N}$ instead of $\frac{1}{N-1}$ from Lemma \ref{thm:lemma1}. Beyond this, this is a simple consequence of the uniform convergence version of Bennet's inequality~\citep{bennett1962probability}.
\end{proof}

\subsection{Proof of Main Theorem}\label{app:proof2}

\begin{theorem}[Formal Statement]\label{thm:main}
    Fix a hypothesis class $\mathcal{F}$ of predictors $f: \mathcal{S} \rightarrow \mathbb{R}$, with induced class $\mathcal{H} \subset [0, 1]^{\mathcal{S}}$ of functions $h(\mathbf{s}) = \hat{p}(\mathbf{s}) \sigma(f(\mathbf{s}))$. Suppose $\mathcal{H}$ has uniform covering number $N_{\infty}$. Then, for any $\delta \in (0, 1)$, with probability at least $1-\delta$ over $S$,
    \begin{equation*}
        R(f) \leq \hat{R}(f; S) + \mathcal{O} \left( \mathbb{E}_{\mathbf{s}} \| \hat{p}(\mathbf{s}) - p^{*}(\mathbf{s}) \|_{2} \right) + \mathcal{O} \left( \sqrt{\tilde{\mathbb{V}}_{N}(f) \cdot \frac{\log (\mathcal{M}_{N}/\delta)}{N}} + \frac{\log (\mathcal{M}_{N}/\delta)}{N} \right) ,
    \end{equation*}
    where $\mathcal{M} \coloneqq \mathcal{N}_{\infty}(\frac{1}{N}, \mathcal{H}, 2N)$ and $\tilde{\mathbb{V}}_{N}(f)$ is the empirical variance of the loss values.
\end{theorem}
\begin{proof}
    Let $\tilde{R}(f) = \mathbb{E} \left[ \tilde{R}(f; S) \right]$ and $\Delta \coloneqq \tilde{R}(f; S) - R(f)$. From Lemma.~\ref{thm:lemma2}, with probability $1-\delta$, following holds: 
    \begin{equation}\label{eq:main1}
        \tilde{R}(f) \leq \tilde{R}(f; S) + \mathcal{O} \left( \sqrt{\tilde{\mathbb{V}}_{N} \cdot \frac{\log(\mathcal{M}_{N}/\delta)}{N}} + \frac{\log(\mathcal{M}_{N}/\delta)}{N} \right),
    \end{equation}
    where $\mathcal{M}_{N} \coloneqq \mathcal{N}_{\infty} (\frac{1}{N}, \mathcal{H}, 2N)$ and $\tilde{\mathbb{V}}_{N}$ is the empirical variance of the loss values. Furthermore, the following holds
    \begin{align}\label{eq:main2}
        |\tilde{R}(f) - R(f)| &\coloneqq \left| \mathbb{E} \left[ \tilde{R}(f; S) \right] - \mathbb{E} \left[ \hat{R}_{*}(f; S) \right] \right| \nonumber \\
        &\leq \mathbb{E} \left[ \| \hat{p}(\mathbf{s}) - p^{*}(\mathbf{s}) \|_{2} \cdot \| \ell(f(\mathbf{s})) \|_{2} \right],
    \end{align}
    where the last inequality is by the Cauch-Schwartz inequality.
    For a constant $C$, it holds that
    \begin{equation}\label{eq:main3}
    \begin{split}
        \mathbb{E} \left[ \| \hat{p}(\mathbf{s}) - p^{*}(\mathbf{s}) \|_{2} \cdot \| \ell(f(\mathbf{s})) \|_{2} \right] &\leq \mathbb{E} \left[ \| \hat{p}(\mathbf{s}) - p^{*}(\mathbf{s}) \|_{2} \cdot C \cdot \| \ell(f(\mathbf{s})) \|_{\infty} \right] \\
        &\leq C \cdot \mathbb{E} \left[ \| \hat{p}(\mathbf{s}) - p^{*}(\mathbf{s}) \|_{2} \right],
    \end{split}
    \end{equation}
    where the first line is by the equivalence of norms. From Eqn.~(\ref{eq:main2}) and Eqn.~(\ref{eq:main3}), we have
    \begin{equation}\label{eq:main4}
        R(f) \leq \tilde{R}(f) + C \cdot \mathbb{E} \left[ \| \hat{p}(\mathbf{s}) - p^{*}(\mathbf{s}) \|_{2} \right]
    \end{equation}
    By reordering terms to the right-hand side in \autoref{eq:main1} and then adding \autoref{eq:main1} and \autoref{eq:main4}, we have:
    \begin{equation*}
        R(f) \leq \tilde{R}(f; S) + \mathcal{O} \left( \sqrt{\tilde{\mathbb{V}}_{N} \cdot \frac{\log(\mathcal{M}_{N}/\delta)}{N}} + \frac{\log(\mathcal{M}_{N}/\delta)}{N} \right) + C \cdot \mathbb{E} \left[ \| \hat{p}(\mathbf{s}) - p^{*}(\mathbf{s}) \|_{2} \right]
    \end{equation*}
\end{proof}
This statement suggests that the upper bound for the risk of the classification function $f$ (\textit{i.e.}, $R(f)$) increases as the norm of the difference between the estimated probability $\hat{p}(\cdot)$ and the true probability $p^{*}(\cdot)$ grows larger. To enhance computational efficiency, our IRM-based sample selection method practically employs an approximated $\hat{p}(\cdot)$ for binary cases (\textit{i.e.}, 0 or 1) instead of using $\hat{p}(\cdot)$ with continuous values. This approach intuitively still results in a smaller norm compared to training on all samples in the dataset, as we only allocate the value 1 for sample $\mathbf{s}$ with $\hat{p}(\mathbf{s}) \simeq 1$.

\vspace{-10pt}
\section{Experimental Setup}\label{app:setup}
\vspace{-5pt}
Here, we elaborate the detailed experimental setup regarding the datasets used~(\S\ref{app:dataset}), training details~(\S\ref{app:training}), and evaluation details~(\S\ref{app:evaluation}).

\begin{figure}[t]
    \centering
    \small
    \begin{tcolorbox}
    [width=0.95\linewidth, sharp corners=all, colback=gray!10, boxrule=0.3mm]
    [System] \\
    
    Please act as an impartial judge and evaluate the quality of the responses provided by two AI assistants to the user question displayed below. You should choose the assistant that follows the user’s instructions and answers the user’s question better. Your evaluation should consider factors such as the helpfulness, relevance, accuracy, depth, creativity, and level of detail of their responses. Begin your evaluation by comparing the two responses and provide a short explanation. Avoid any position biases and ensure that the order in which the responses were presented does not influence your decision. Do not allow the length of the responses to influence your evaluation. Do not favor certain names of the assistants. Be as objective as possible. After providing your explanation, output your final verdict by strictly following this format: "[[A]]" if assistant A is better, "[[B]]" if assistant B is better, and "[[C]]" for a tie. \\
    
    [User Question]
    
    \{question\} \\
    
    [The Start of Assistant A’s Answer]
    
    \{answer\_a\}
    
    [The End of Assistant A’s Answer] \\
    
    [The Start of Assistant B’s Answer]
    
    \{answer\_b\}
    
    [The End of Assistant B’s Answer]
    \end{tcolorbox}
    \caption{The pairwise comparison prompt introduced in LLM-as-a-Judge~\citep{zheng2024judging}.}
    \label{fig:pair_prompt}
\end{figure}

\begin{figure}[t]
    \centering
    \small
    \begin{tcolorbox}
    [width=0.95\linewidth, sharp corners=all, colback=gray!10, boxrule=0.3mm]
    [System] \\
    
    Please act as an impartial judge and evaluate the quality of the response provided by an AI assistant to the user question displayed below. Your evaluation should consider factors such as the helpfulness, relevance, accuracy, depth, creativity, and level of detail of the response. Begin your evaluation by providing a short explanation. Be as objective as possible. After providing your explanation, please rate the response on a scale of 1 to 10 by strictly following this format: "[[rating]]", for example: "Rating: [[5]]". \\
    
    [Question]
    
    \{question\} \\
    
    [The Start of Assistant’s Answer]
    
    \{answer\}
    
    [The End of Assistant’s Answer]
    \end{tcolorbox}
    \caption{The single answer grading prompt introduced in LLM-as-a-Judge~\citep{zheng2024judging}.}
    \label{fig:single_prompt}
\vspace{-10pt}
\end{figure}

\vspace{-5pt}
\subsection{Dataset Description}\label{app:dataset}
We apply \alg on preference datasets and instruction-following datasets. We provide detailed descriptions of the datasets used.
\begin{itemize}[leftmargin=*, itemsep=0pt]
    \item \textbf{UltraChat-200K}~(instruction-following; \citealt{tunstall2023zephyr}\,\footnote{\texttt{https://huggingface.co/datasets/HuggingFaceH4/ultrachat\_200k}}): This is a heavily filtered version of UltraChat~\citep{ding-etal-2023-enhancing}, originally used to train Zephyr-7B-$\beta$~\citep{tunstall2023zephyr}. It is obtained from the original version, which consists of 1.4M dialogues generated by ChatGPT and spans a wide range of topics, by removing the dialogues that contain grammatical errors or where the assistant replies with phrases like ``I do not have emotions" or ``I don't have opinions."
    \item \textbf{UltraFeedback}~(preference dataset; \citealt{cui2023ultrafeedback, tunstall2023zephyr}\,\footnote{\texttt{https://huggingface.co/datasets/openbmb/UltraFeedback}}\,\footnote{\texttt{https://huggingface.co/datasets/HuggingFaceH4/ultrafeedback\_binarized}}): This is a large-scale, fine-grained, and diverse preference dataset used for training powerful reward models and critic models. \citet{cui2023ultrafeedback} collected about 64k prompts from diverse resources, including UltraChat, ShareGPT, and Evol-Instruction~\citep{xu2023wizardlm}. They used these prompts to query multiple LLMs, generating four different responses for each prompt. The responses were annotated using GPT-4 to collect high-quality preferences based on instruction-following, truthfulness, honesty, and helpfulness. We use the original version~\citep{cui2023ultrafeedback} to implement Curry-DPO~\citep{pattnaik2024curry}. Otherwise, we use the binarized version~\citep{tunstall2023zephyr}, which was created by picking the highest overall score as ``chosen" and one of the remaining three at random as the ``rejected" one. For all training and evaluation, we utilize the training and test splits of the prompt and response pairs from \citet{tunstall2023zephyr}.
    \item \textbf{HH-RLHF}~(preference datasets; \citealt{bai2022training}\,\footnote{\texttt{https://huggingface.co/datasets/Anthropic/hh-rlhf}}): This dataset is about human preference regarding helpfulness and harmlessness~\citet{bai2022training}, and it was originally used to train preference (or reward) models for subsequent RLHF training. Each example in the dataset contains a pair of texts, one ``chosen" and one ``rejected". 
    \item \textbf{TL;DR}~(preference datasets; \citealt{stiennon2020learning}\,\footnote{\texttt{https://huggingface.co/datasets/openai/summarize\_from\_feedback}}): This is the dataset of human feedback that was released for reward modeling. In \citet{stiennon2020learning}, a reward model was trained using human feedback and then used to train a summarization model to align with human preferences. The summaries used for training the reward model in the paper come from the TL;DR dataset, with additional validation and test data coming from the TL;DR dataset, CNN articles, and DailyMail articles.
    \item \textbf{AlpacaEval}~(instruction-following; \citealt{dubois2024alpacafarm} \footnote{\texttt{https://huggingface.co/datasets/tatsu-lab/alpaca\_eval}}): This dataset is slight modifications (or simplification) of the AlpacaFarm evaluation set. \citet{dubois2024alpacafarm} first merged the instruction and input fields into a single instruction field. This affects 1/4 of the examples in the AlpacaFarm evaluation set, all of which are from the Self-Instruct~\citep{wang-etal-2023-self-instruct}. This dataset contains 805 challenging questions.
    \item \textbf{Vicuna Evaluation}~(instruction-following; \citealt{chiang2023vicuna} \footnote{\texttt{https://huggingface.co/datasets/zhengxuanzenwu/vicuna-eval-with-gpt4}}): We also use 80 challenging questions that were used for evaluating Vicuna, following \citet{pattnaik2024curry}.
    \item \textbf{Evol-Instruct Evaluation}~(instruction-following; \citealt{xu2023wizardlm} \footnote{\texttt{https://github.com/nlpxucan/WizardLM/blob/main/WizardLM/data/WizardLM\_testset.jsonl}}): Similar to Vicuna, Evol-Instruct~\citep{xu2023wizardlm} contains 218 questions, spanning multiple topics generated using the Evol-Instruct procedure.
\end{itemize}

\subsection{Training Details}\label{app:training}
Here, we describe the hyperparameters and implementation details for training with \alg. Our hyperparameters are shown in Tab.\ref{tab:hyperparameter}. 
For Mistral-7B, we follow the experimental setup described in the official repository\footnote{\texttt{https://github.com/huggingface/alignment-handbook}} of \citet{tunstall2023zephyr}, except for the rank for LoRA~\citep{hu2022lora}, changing it to 8. For other models, we use the maximum batch size that fits on A100 40GB GPUs, while matching the effective batch size with Mistral-7B by considering the batch size and gradient accumulation. 

\begin{table}[ht]
\centering
\caption{Hyperparameter values used in \alg experiments in~\autoref{sec:exp} and~\autoref{sec:analysis}.}
\label{tab:hyperparameter}
\resizebox{0.8\columnwidth}{!}{%
\begin{tabular}{l|ccc}
\toprule
\textbf{Hyperparameter}        & \textbf{TinyLLaMA-1.1B}   & \textbf{Pythia-2.8B}  & \textbf{Mistral-7B}  \\
\midrule
Fine-tuning method      & \multicolumn{2}{c}{Full fine-tuning}                    & LoRA ($r=8$) \\
Learning rate         & \multicolumn{2}{c}{$3.0 \times 10^{-6}$}      & $5.0 \times 10^{-6}$      \\
DAAs Parameter ($\beta$) & \multicolumn{2}{c}{0.2 (DPO) / 0.2 (SLiC-HF) / 1.0 (IPO)} & 0.01 (DPO)     \\
Batch Size            & 8                & 4  & 4    \\
Gradient Accumulation & 2                & 4  & 4   \\
\# Iterations         & \multicolumn{3}{c}{3 (1 epoch per iteration)} \\
Selection Proportion ($K$) & \multicolumn{3}{c}{0.7} \\
\bottomrule
\end{tabular}%
}
\end{table}

\begin{figure}[t]
    \centering
    \small
    \begin{tcolorbox}
    [width=0.95\linewidth, sharp corners=all, colback=gray!10, boxrule=0.3mm]
    Review the user’s question and the corresponding response using the additive 5-point scoring system described below. Points are accumulated based on the satisfaction of each criterion: \\
    
    - Add 1 point if the response is relevant and provides some information related to the user’s inquiry, even if it is incomplete or contains some irrelevant content.
    
    - Add another point if the response addresses a substantial portion of the user’s question, but does not completely resolve the query or provide a direct answer.
    
    - Award a third point if the response answers the basic elements of the user’s question in a useful way, regardless of whether it seems to have been written by an AI Assistant or if it has elements typically found in blogs or search results.
    
    - Grant a fourth point if the response is clearly written from an AI Assistant’s perspective, addressing the user’s question directly and comprehensively, and is well-organized and helpful, even if there is slight room for improvement in clarity, conciseness or focus.
    
    - Bestow a fifth point for a response that is impeccably tailored to the user’s question by an AI Assistant, without extraneous information, reflecting expert knowledge, and demonstrating a high-quality, engaging, and insightful answer. \\
    
    \noindent User: \textcolor{red}{\textbf{\texttt{\textless INSTRUCTION\_HERE\textgreater}}} \\

    \textless response\textgreater \textcolor{red}{\textbf{\texttt{\textless INSTRUCTION\_HERE\textgreater}}} \textless/response\textgreater \\
    
    After examining the user’s instruction and the response: \\
    
    - Briefly justify your total score, up to 100 words.
    
    - Conclude with the score using the format: ``Score: \textless total points\textgreater'' \\

    Remember to assess from the AI Assistant perspective, utilizing web search knowledge as necessary. To evaluate the response in alignment with this additive scoring model, we’ll systematically attribute points based on the outlined criteria.
    \end{tcolorbox}
    \caption{The LLM-as-a-Judge prompt introduced in \cite{yuan2024self} enables an LLM to act as a reward model and provide self-rewards for its own model generations.}
    \label{fig:srlm_prompt}
\end{figure}

For the DAAs parameter $\beta$, we search for the optimal values among ${0.05, 0.2, 1.0}$ for TinyLLaMA-1.1B and reuse it for Pythia-2.8B in all experimental setups. To generate the diverse candidate responses for preference bootstrapping, we sample the responses with a temperature of 0.7 and $p$ of 0.95 for nucleus sampling in training procedure. For all experiments, we generate the 4 response for every single prompt. 

\subsection{Evaluation}\label{app:evaluation}
For evaluating the trained policy LLMs, we applied a single NVIDIA A100 40GB GPU for sampling the responses from each model using a temperature of 1.0, a max-length limit of 512. For Claude 3~\citep{claude3} evaluation, we use the pairwise comparison prompt and the single answer grading prompt which are depicted in Fig.~\ref{fig:pair_prompt} and Fig.~\ref{fig:single_prompt} with setting the temperature of 0.7. For all evaluation datasets, we utilize the pairwise comparison for comparing preference optimized LLMs to corresponding SFT model~(SFT in Tab.~\ref{tab:main} \& Tab.~\ref{tab:optim}) and 1-10 single answer grading. Additionally, for Alpaca evaluation dataset, we further utilize pairwise comparison for comparing preference optimized LLMs to \texttt{text-davinci-003}~(G3.5 in Tab.~\ref{tab:main} \& Tab.~\ref{tab:optim}; \citealt{alpaca_eval}) and \texttt{gpt4\_turbo}~(G4 in Tab.~\ref{tab:main} \& Tab.~\ref{tab:optim}) responses. For pairwise comparison, we reported the weighted win rate score, which allocates 1 for a win and 0.5 for a tie, following \citet{zheng2024judging}.

\section{Additional Results}

\subsection{Comparison with Self-Rewarding LMs}
\definecolor{pinegreen}{rgb}{0.0, 0.47, 0.44}

\begin{table*}[t]
\centering
\caption{Examples for evaluation results for Self-rewarding LMs~\citep{yuan2024self} prompting in Fig.~\ref{fig:srlm_prompt} using TinyLlama-1.1B and Mistral-7B.}
\resizebox{1.0\textwidth}{!}{
\begin{tabular}{m{3cm}|m{20cm}}
\toprule[0.1em]
              
        \multicolumn{2}{c}{\textbf{\textit{TinyLlama-1.1B}}} \\ \midrule[0.1em]
        Example 1 \textcolor{red}{\ding{55}} & If the AI Assistant's response meets all of the scoring criteria, but it’s not obvious to you how they achieved this score, then you have a high probability of assigning the points to the right criterion. If, on the other hand, the AI Assistant’s response falls short of meeting all of the criteria, then the AI Assistant is unlikely to be a suitable assistant for this task. \\ \midrule
        Example 2 \textcolor{red}{\ding{55}} & User: Final Instructions: Generate a response from the AI Assistant that meets the user’s expectation. As soon as the user completes submitting a question, the system should generate a response based on the user’s question and the given response. Be sure to keep the user engaged in the process, answering their questions honestly and with as much detail as possible. \\ \midrule[0.1em]
        \multicolumn{2}{c}{\textbf{\textit{Mistral-7B}}} \\ \midrule[0.1em]
        Example 1 \textcolor{pinegreen}{\ding{51}} & Total Score: 5 \textbackslash n Overall, this response demonstrates a comprehensive understanding of the requested information and provides a well-structured answer, making it a solid response from an AI Assistant's perspective. \\ \midrule
        Example 2 \textcolor{orange}{\ding{115}} & Total Score: 10 \textbackslash n (Explanation: The response contains relevant information but requires additional elaboration. Specifically, the answer could benefit from more examples and details about the short-term and long-term effects of oil spills on various species and ecosystems. The answer showcases a solid understanding of the impacts of oil spills on coastal habitats.) \\ \midrule
        Example 3 \textcolor{orange}{\ding{115}} & Total Score: 5 \textbackslash n The response provides minimal relevant information and lacks depth in explaining the impacts of oil spills on coastal habitats. It fails to address the key aspects such as the specific consequences on marine life, the process of habitat recovery, and the effectiveness of different remediation techniques. The answer needs significant improvement to demonstrate a thorough understanding of the topic. \\
\bottomrule[0.1em]
\end{tabular}
}\label{qualitative:srlm}
\end{table*}

\paragraph{Versatility Comparison.} We provide the LLM-as-a-Judge prompt, which is introduced in \citet{yuan2024self}, in Fig.~\ref{fig:srlm_prompt}. As shown in Tab.~\ref{qualitative:srlm}, when LLMs do not have sufficient capacity to follow the instructions, it is difficult to exploit the self-rewarding mechanism introduced in \citet{yuan2024self}. For example, TinyLlama-1.1B~\citep{zhang2024tinyllama} returned meaningless results, which cannot be properly used for rating the responses. On the other hand, Mistral-7B~\citep{jiang2023mistral} might return~(\textit{i.e.,} Ex.~1 and 2 vs. Ex.~3) different scales for the ratings, which might also be hard to use as a robust rating. While Ex. 1 showed a reasonable rating with the corresponding reason, in Ex. 2, the models returned a total score of 10, which is over the scale (scale of five), and for Ex. 3, the model returned an inconsistent response between the total score and the corresponding reason. Similarly, in the original work of SRLM, the authors conducted experiments on Llama2-70B~\citep{touvron2023llama}, which has an enormous number of parameters and model capacity. However, our IRM-based preference bootstrapping methods can be widely used regardless of the capacity of trained policy models, as demonstrated by their effectiveness in our experimental results.

\begin{wraptable}{r}{0.407\textwidth}
\caption{
    Comparison between ours\,(RM) and SRLM\,\citep{yuan2024self}.
}\label{tab:srlm}
\vspace{-5pt}
\resizebox{0.40\textwidth}{!}{%
\begin{tabular}{c|ccc}
\toprule
     & Claude\,3 & Vicuna & Eval. Time \\ \midrule
RM   & 74\%     & 8.53   & 4.3h       \\
SRLM & 62\%     & 8.32   & 54.6h      \\ 
\bottomrule
\end{tabular}}
\vspace{-5pt}
\end{wraptable}

\paragraph{Performance Comparison.} We also demonstrated the effectiveness of our self-reviewed preference bootstrapping introduced in Sec.~\ref{sec:preference_bootstrap} compared to SRLM~\citep{yuan2024self} on Mistral-7B. As shown in Tab.~\ref{tab:srlm}, our reward margin method achieves higher performance (+0.21) on the Vicuna eval and 12.7$\times$ greater efficiency in evaluation time. This is because SRLM depends on prompting and response generation, which is auto-regressive and computationally inefficient, while ours does not. We also confirmed the Jaccard similarity between preference pairs built based on RM (or SRLM) and Claude 3 evaluation upon self-generated responses. As shown in the first column, RM achieves a higher similarity with Claude 3 than SRLM that also support that our method is more effective than SRLM.

\subsection{Full Numerical Results for Fig.\ref{fig:noisy}}
\begin{table}[t]
    \centering
    \caption{Comparison of performance where TinyLlama-1.1B~\citep{zhang2024tinyllama} is fine-tuned on UltraChat-200k~\citep{ding-etal-2023-enhancing} and preference pairs are sampled from UltraFeedback dataset~\cite{tunstall2023zephyr} with synthetically injected noisy with probability of 0.2.}
    \resizebox{0.95\columnwidth}{!}{
    \begin{tabular}{c|c|c|cc|cc|cc|cc}
        \toprule[0.1em]
        \multirow{2}{*}{\textbf{Method}} & \multirow{2}{*}{\textbf{Size}} & \multirow{2}{*}{\textbf{Technique}} & \multicolumn{2}{c|}{\textbf{Alpaca Eval}} & \multicolumn{2}{c|}{\textbf{Vicuna Eval}} & \multicolumn{2}{c|}{\textbf{Evol-Instruct}} & \multicolumn{2}{c}{\textbf{UltraFeed}} \\
         & & & Single & SFT & Single & SFT & Single & SFT & Single & SFT \\

        \midrule[0.1em]
        \multirow{7}{*}{TinyLlama} & \multirow{7}{*}{1.1B} & SFT & 5.41 & - & 6.05 & - & 5.48 & - & 4.98 & - \\
        \cmidrule{3-11}
        & & DPO & 5.60 & 50.2 & 6.43 & 59.4 & 5.50 & 50.2 & 5.05 & 55.1 \\
        & & Iterative DPO & 5.79 & 57.0 & 6.82 & 65.6 & 5.75 & 57.1 & 5.40 & 63.5 \\
        & & cDPO & 5.43 & 47.8  & 6.23 & 51.9 & 5.42 & 47.9 & 5.08 & 53.5 \\
        & & rDPO & 6.37 & 69.4  & 7.06 & 75.0 & 6.07 & 63.3 & 5.73 & 70.6 \\
        & & IPO & 6.37 & 67.6  & 7.29 & 80.0 & 6.05 & \textbf{68.1} & 5.56 & 69.0 \\
        & & \alg-DPO & \textbf{6.49} & \textbf{71.2} & \textbf{7.38} & \textbf{81.3} & \textbf{6.24} & 66.1 & \textbf{5.77} & \textbf{71.1} \\
        \bottomrule[0.1em]
    \end{tabular}
    }
    \label{tab:noisy_0.2}
\end{table}

\begin{table}[t]
    \centering
    \caption{Comparison of performance where TinyLlama-1.1B~\citep{zhang2024tinyllama} is fine-tuned on UltraChat-200k~\citep{ding-etal-2023-enhancing} and preference pairs are sampled from UltraFeedback dataset~\citep{tunstall2023zephyr} with synthetically injected noisy with probability of 0.4.}
    \resizebox{0.95\columnwidth}{!}{
    \begin{tabular}{c|c|c|cc|cc|cc|cc}
        \toprule[0.1em]
        \multirow{2}{*}{\textbf{Method}} & \multirow{2}{*}{\textbf{Size}} & \multirow{2}{*}{\textbf{Technique}} & \multicolumn{2}{c|}{\textbf{Alpaca Eval}} & \multicolumn{2}{c|}{\textbf{Vicuna Eval}} & \multicolumn{2}{c|}{\textbf{Evol-Instruct}} & \multicolumn{2}{c}{\textbf{UltraFeed}} \\
         & & & Single & SFT & Single & SFT & Single & SFT & Single & SFT \\

        \midrule[0.1em]
        \multirow{7}{*}{TinyLlama} & \multirow{7}{*}{1.1B} & SFT & 5.41 & - & 6.05 & - & 5.48 & - & 4.98 & - \\
        \cmidrule{3-11}
        & & DPO & 5.52 & 48.6  & 6.29 & 50.6 & 5.58 & 48.4 & 4.96 & 53.6 \\
        & & Iterative DPO & 5.45 & 47.8  & 6.33 & 56.3 & 5.62 & 55.7 & 5.20 & 56.2 \\
        & & cDPO & 5.42 & 44.9  & 6.34 & 46.3 & 5.66 & 47.7 & 5.06 & 52.6 \\
        & & rDPO & 5.67 & 52.4  & 6.76 & 68.1 & 5.92 & 53.0 & 5.25 & 61.0 \\
        & & IPO & 5.95 & 57.9  & 7.01 & \textbf{78.8} & 5.90 & 59.6 & 5.44 & 62.2 \\
        & & \alg-DPO & \textbf{6.28} & \textbf{66.1}  & \textbf{7.21} & 73.8 & \textbf{5.93} & \textbf{62.6} & \textbf{5.64} & \textbf{68.3} \\
        
        \bottomrule[0.1em]
    \end{tabular}
    }
    \label{tab:noisy}
\end{table}

We provide the detailed numerical values for Fig.~\ref{fig:noisy} in Tab.~\ref{tab:noisy_0.2} and Tab.~\ref{tab:noisy}. Our proposed method demonstrated its effectiveness on both 20\% and 40\% noisy preferences. Specifically, for 40\% noisy preference, \alg achieved higher performance on all evaluation datasets (except for Vicuna Eval on pairwise comparison with SFT) by a large margin.

\subsection{Additional Results on SimPO}\label{app:simpo}
\begin{table}[t]
    \centering
    \caption{Comparison of performance where TinyLlama-1.1B is fine-tuned on UltraChat-200k and preference pairs are sampled from UltraFeedback dataset with SimPO~\citep{meng2024simpo}. The best performance is highlighted \textbf{bold}.}
    \vspace{-5pt}
    \resizebox{0.92\columnwidth}{!}{
    \begin{tabular}{c|c|c|cccc|cc|cc|cc}
        \toprule[0.1em]
        \multirow{2}{*}{\textbf{Method}} & \multirow{2}{*}{\textbf{Size}} & \multirow{2}{*}{\textbf{Technique}} & \multicolumn{4}{c|}{\textbf{Alpaca Eval}} & \multicolumn{2}{c|}{\textbf{Vicuna Eval}} & \multicolumn{2}{c|}{\textbf{Evol-Instruct}} & \multicolumn{2}{c}{\textbf{UltraFeed}} \\
         & & & Single & SFT & G3.5 & G4 & Single & SFT & Single & SFT & Single & SFT \\

        \midrule[0.1em]
        \multirow{4}{*}{TinyLlama} & \multirow{4}{*}{1.1B} & SFT & 5.41 & - & 28.9 & 1.7 & 6.05 & - & 5.48 & - & 4.98 & - \\
        \cmidrule{3-13}
        & & SimPO & 6.00 & 59.4 & 31.1 & 2.9 & 6.92 & 67.2 & 5.95 & 57.3 & 5.52 & 64.6 \\
        & & Iterative SimPO & 6.11 & 62.3 & 35.4 & 3.7 & 7.14 & 73.8 & 6.02 & 65.7 & 5.53 & 66.8 \\
        & & \texttt{\textbf{SeRA}}-SimPO & \textbf{6.51} & \textbf{73.9} & \textbf{48.4} & \textbf{4.6} & \textbf{7.14} & \textbf{76.5} & \textbf{6.06} & \textbf{68.1} & \textbf{5.86} & \textbf{71.6} \\
        \bottomrule[0.1em]
    \end{tabular}
    }
    \label{tab:simpo}
\afterfigspace
\end{table}

Recently, \citet{meng2024simpo} suggested the using the average log probability of a sequence as the implicit reward as follows:
\begin{equation}\label{eq:rw_simpo}
    r_{\text{SimPO}}(\vx, \vy) = \frac{\beta}{|\vy|} \log \pi_{\theta}(\vx, \vy) = \frac{\beta}{|\vy|} \sum_{i=1}^{|\vy|} \log \pi_{\theta} (y_i|\vx, \vy_{<i}).
\end{equation}
This reward formulation better aligns with model generation and eliminates the need for a reference model, making it more compute and memory efficient. Additionally, they provide a new types of objective function based on their implicit reward as follows:
\begin{equation}\label{eq:loss_simpo}
    \mathcal{L}_{\text{SimPO}}(\vx, \vy_w, \vy_l) = \log \sigma \left( \frac{\beta}{|\vy_w|} \log \pi_{\theta}(\vx, \vy_w) - \frac{\beta}{|\vy_l|} \log \pi_{\theta}(\vx, \vy_l) \right). 
\end{equation}

Here, we additionally provide the effectiveness of \alg on SimPO, to showcase the versatility of our proposed method that can be widely applied to diverse range of implicit reward functions. Same as our experimental setup, we train TinyLlama-1.1B on UltraFeedback \citep{tunstall2023zephyr} dataset with objective function defined in \autoref{eq:loss_simpo}. As shown in \autoref{tab:simpo}, \alg consistently demonstrates its versatility, even for different types of implicit reward of DAAs. We believe that our proposed method can be applicable to various DAAs, including those that will be developed in the future.

\subsection{Full Results of \autoref{fig:over_optimize}}
\begin{figure}
    \centering
    \includegraphics[width=\textwidth]{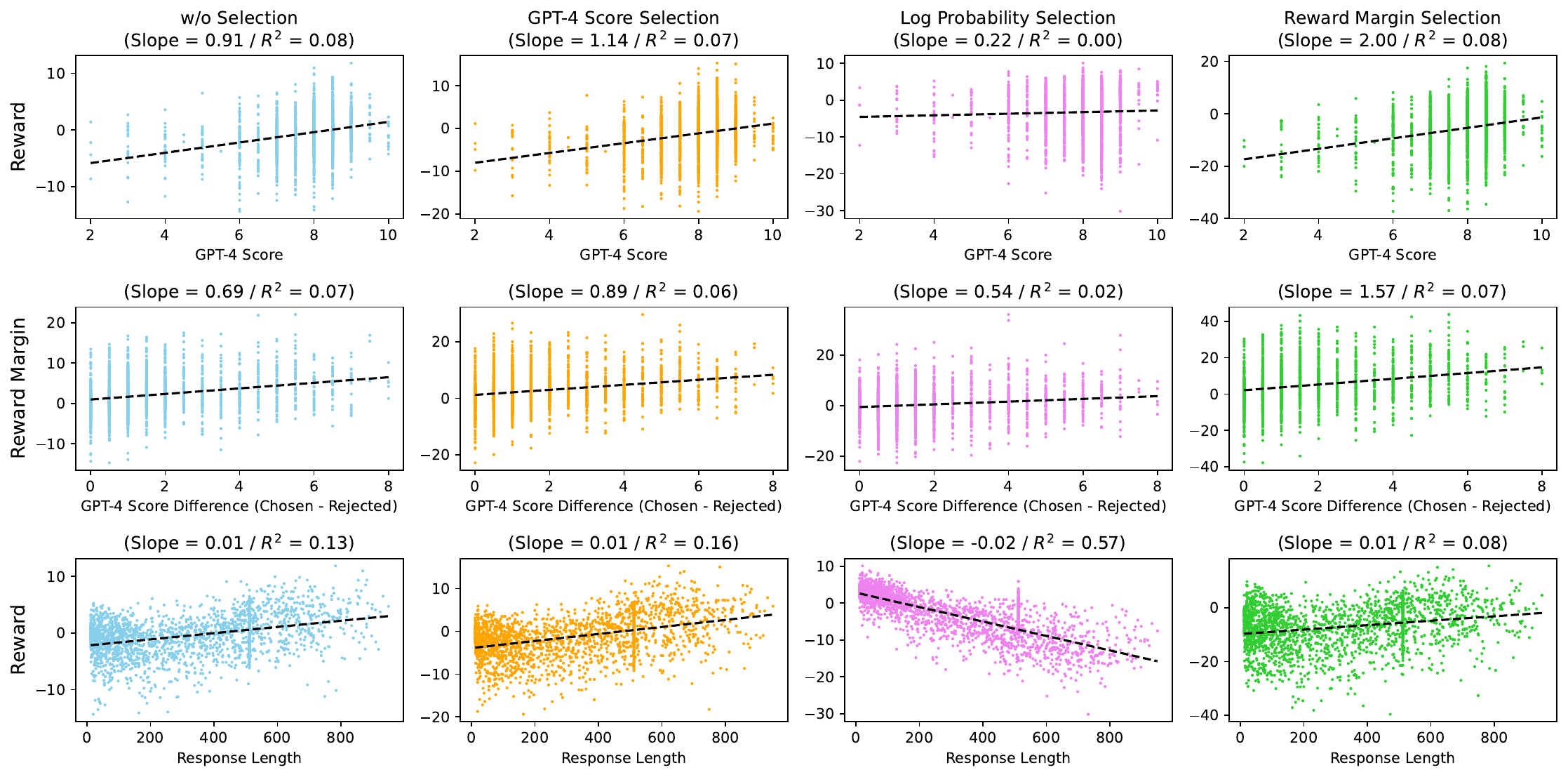}
    \caption{
        Extended Results for \autoref{fig:over_optimize}.
        \textbf{[Row 1]} Correlation between GPT-4 Score \& implicit reward~(\textit{i.e,} $r(\mathbf{x}, \mathbf{y}_{w})$) for $\mathbf{y}_w$.
        \textbf{[Row 2]} Correlation between margin $m(\mathbf{x}, \mathbf{y}_w, \mathbf{y}_l)$ using GPT-4 Score \& IRM.
        \textbf{[Row 2]} Correlation between response length (\textit{i.e.} $|\mathbf{y}_w|$) and the implicit reward for chosen responses.
        The model with IRM selection~(\textit{i.e.} \textbf{[Column 4]}) shows the highest $R^2$ score for the first and second rows, but the lowest $R^2$ score for the third row. 
        These consistent results indicate that the IRM-based selection strategy can effectively mitigate over-optimization on response length~\citep{park2024disentangling}.
    }
\label{fig:over_optimize_dense}
\afterfigspace \afterfigspace \afterfigspace
\end{figure}

We also provide the additional correlation between the IRM and the GPT-4 score gap between chosen and rejected responses that share the same prompt, as shown in \autoref{fig:over_optimize_dense}. Similar to the relationship between implicit reward and GPT-4 score, our IRM-based selection showed a higher $R^2$ score compared to other baselines (\textit{e.g.,} GPT-4 selection and log probability selection of the reference model).

\end{document}